\documentclass[conference]{IEEEtran}

\usepackage{amsmath,amsthm,amsfonts}
\usepackage{amsmath}
\usepackage{algorithmic}
\usepackage[ruled, lined, vlined, linesnumbered]{algorithm2e}
\usepackage{array}
\usepackage{booktabs}
\usepackage{epstopdf}
\usepackage{rotating}
\usepackage{caption}
\usepackage{subcaption}
\usepackage{tabularx}
\usepackage{threeparttable}
\usepackage{multirow}
\usepackage{multicol}
\usepackage{nccmath}
\usepackage{xcolor}
\usepackage{xspace}
\usepackage{pifont}
\usepackage{graphicx}  
\usepackage{epstopdf}
\usepackage{graphicx}

\newtheorem{theorem}{Theorem}
\newtheorem{lemma}{Lemma}

\newtheorem{definition}{Definition}
\newtheorem{example}{Example}

\newcommand{\aka}{\emph{a.k.a.}\xspace}
\newcommand{\etal}{\emph{et~al.}\xspace}
\newcommand{\eg}{\emph{e.g.},\xspace}
\newcommand{\ie}{\emph{i.e.},\xspace}
\newcommand{\etc}{etc.\xspace}

\newcommand\figref[1]{\figurename~\ref{#1}}

\newcommand\tabref[1]{Table~\ref{#1}}

\newcommand\secref[1]{Sec.~\ref{#1}}
\newcommand\equref[1]{Eq.~(\ref{#1})}

\newcommand\algref[1]{Alg.~\ref{#1}}
\newcommand\defref[1]{Def.~\ref{#1}}
\newcommand\expref[1]{{Example~\ref{#1}}}
\newcommand\lemref[1]{Lemma~\ref{#1}}
\newcommand\thmref[1]{Theorem~\ref{#1}}

\newcommand{\fakeparagraph}[1]{\vspace{1mm}\noindent\textbf{#1.}}

\definecolor{dc1}{HTML}{B4FFA4}
\definecolor{dc2}{HTML}{FBFB0C}
\definecolor{dc3}{HTML}{F57171}
\definecolor{dc4}{HTML}{F7F6F6}

\ifodd 0
\newcommand{\rev}[1]{{\color{blue}#1}} 

\else
\newcommand{\rev}[1]{#1}

\fi

\ifCLASSOPTIONcompsoc
  \usepackage[nocompress]{cite}
\else
  \usepackage{cite}
\fi

\begin{document}

\title{Efficient Data Valuation Approximation in Federated Learning: A Sampling-based Approach}


\author{
        Shuyue Wei$^{1}$, Yongxin Tong$^{1}$, Zimu Zhou$^{2}$, Tianran He$^{1}$, Yi Xu$^{1}$ \\
    
    
        $^{1}$ State Key Laboratory of Complex \& Critical Software Environment Lab, School of Computer Science and Engineering, \\
        Beijing Advanced Innovation Center for Future Blockchain and Privacy Computing, Beihang University, China \\
        $^{2}$ City University of Hong Kong, Hong Kong, China \\
        $^{1}$\{weishuyue, yxtong, hetianran, xuy\}@buaa.edu.cn, \quad 
         $^{2}$zimuzhou@cityu.edu.hk
}


    



\maketitle

\begin{abstract}\label{sec:abstract}
    Federated learning (FL) has emerged as a prominent distributed learning paradigm to utilize datasets across multiple data providers.
    In FL, cross-silo data providers often hesitate to share their high-quality dataset unless their data value can be fairly assessed.
    Shapley value (SV) has been advocated as the standard metric for data valuation in FL due to its desirable properties.
    However, the computational overhead of SV is prohibitive in practice, as it inherently requires training and evaluating an FL model across an exponential number of dataset combinations.
    Furthermore, existing solutions fail to achieve high accuracy and efficiency, making practical use of SV still out of reach, because they ignore choosing suitable computation scheme for approximation framework and overlook the property of utility function in FL.
    \rev{We first propose a unified stratified-sampling framework for two widely-used schemes.
    Then, we analyze and choose the more promising scheme under the FL linear regression assumption.
    After that,} we identify a phenomenon termed {key combinations}, where only limited dataset combinations have a high-impact on final data value.
    Building on these insights, we propose a practical approximation algorithm, \textsl{IPSS}, which strategically selects high-impact dataset combinations rather than evaluating all possible combinations, thus substantially reducing time cost with minor approximation error.
    Furthermore, we conduct extensive evaluations on the FL benchmark datasets to demonstrate that our proposed algorithm outperforms a series of representative baselines in terms of efficiency and effectiveness. 

\end{abstract}


\section{Introduction} \label{sec:intro}

In recent years, \textit{federated learning} (FL) has gained increasing attention in both academia and industry, as it provides a novel solution to utilize datasets across multiple data-rich entities without directly accessing raw data \cite{FL_survey_long, yang2019federated, CURS24_FL_Survey, TKDE23_FL_Survey, McMahanMRHA17}.
In FL, cross-silo data providers may be reluctant to share high-quality datasets unless the value of their datasets are fairly measured, ensuring they receive appropriate compensation \cite{Bigdata19_SV, ICDE22_SV_Wang, VLDB23_SVF}.
Therefore, data valuation is a fundamental problem in FL, as it is the prerequisite for motivating multiple data providers to contribute, thereby serving as a crucial component for the sustainability of the FL ecosystem.

The \textit{Shapley value} (SV), a classical concept for measuring the player's contribution in cooperation, has been considered as the standard data valuation metric for FL in prior work \cite{Bigdata19_SV, FLIP20_Wei, FLIP20_FSV, vldb19shapley, ICDE22_SV_Wang, TIST22_GTG, VLDB23_SVF, VLDB24_Shapley}, as it uniquely satisfies several basic fairness properties~\cite{myerson2013game} (\eg \textit{no-free-riders}, \textit{symmetric fairness} and \textit{linear additivity}).
However, the SV-based data valuation is widely-acknowledged computationally prohibitive due to its intrinsic combinatorial nature~\cite{Survey_SV_DB, Survey_SV_ML}, \ie we have to train and evaluate FL models across an exponential number of combinations of datasets, as the toy example in \figref{fig:motivation}(a).

\begin{figure}[t]
        \centering
        \includegraphics[width=\linewidth]{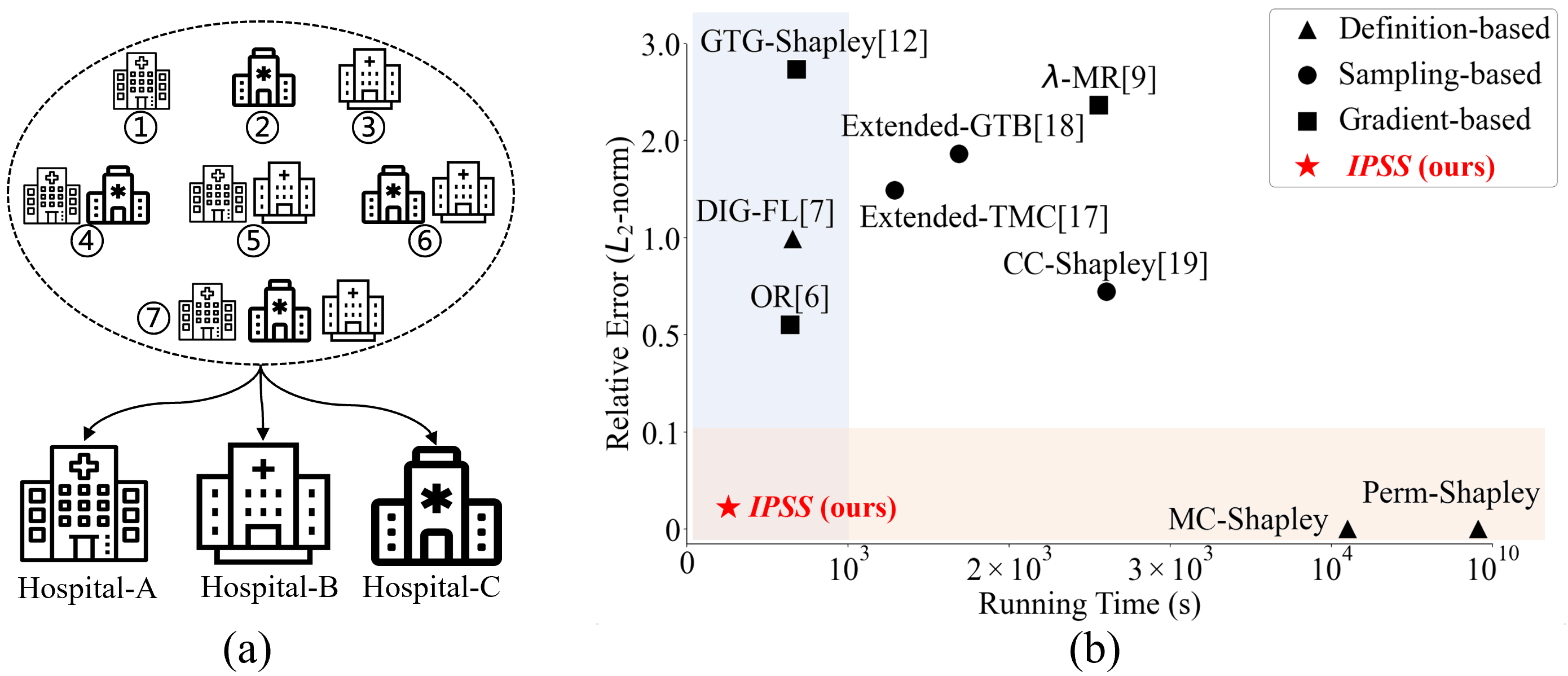}
        \vspace{-0.5em}
        \caption{
        {\small (a):Three hospitals collaborate to train the FL model and aim to identify each hospital's data value. 
        The SV-based data valuation requires training and evaluating FL models across all possible hospital combinations (\ding{172}$\sim$\ding{178}), \ie it needs to tackle seven FL processes.
        As client number increases, the number of required combinations grows exponentially. }
        {\small (b):Evaluations on the FL benchmark dataset FEMNIST with ten FL clients indicate that the existing solutions fail to achieve both high effectiveness and efficiency simultaneously.}
        }
        \label{fig:motivation}
        \vspace{-1.8em}
\end{figure}

    The SV-based data valuation has attracted extensive research interest from both the database and the data mining communities, where researchers prioritize efficiency as a central issue and devise a line of approximation algorithms~\cite{Bigdata19_SV, FLIP20_Wei, vldb19shapley,  FLIP20_FSV, TIST22_GTG, ICDE22_SV_Wang, VLDB23_SVF}.
    Existing solutions can primarily be divided into two categories.
    \textit{(i) The first category is the gradient-based approximation} ~\cite{Bigdata19_SV, FLIP20_Wei, TIST22_GTG}, which utilizes gradients in the training process to construct the FL models, that are required to be evaluated in data valuation, thereby avoiding extra FL training processes.
    Though these solutions provide notable computational efficiency, their lack of accuracy guarantees diminishes their practicality (as in \figref{fig:motivation}(b)).
    \textit{(ii) The second category is the sampling-based approximation}~\cite{icml19shapley, aistats19shapley, SIGMOD23_SV_Zhang}, which only trains and evaluates FL models under a few number of dataset combinations chosen from all possible combinations to estimate the data value.
    Recently, more studies~\cite{Survey_SV_ML, Survey_SV_DB, icml19shapley, aistats19shapley, SIGMOD23_SV_Zhang} have advocated for sampling-based approaches as they provide a flexible trade-off between accuracy and efficiency.
    However, prior sampling-based approximations still fail to achieve both high efficiency and accuracy simultaneously in FL, as shown in \figref{fig:motivation}(b), primarily due to the following two limitations. 

    \textit{- Limitation 1}: \textit{Ignoring selecting the suitable computation scheme of Shapley value}.
    There are two commonly used equivalent expressions of the Shapley value, including the marginal-contribution-based (\textsl{\small MC-SV}) and the complementary-contribution-based (\textsl{\small CC-SV}), each provides a computation scheme by its definition.
    Choosing a more suitable scheme for the approximation framework is often overlooked.
    As shown in \secref{sec:framework_comp_scheme}, taking the \textsl{\small MC-SV} for our sampling framework can yield a lower variance in the approximation.

    \textit{ - Limitation 2}: \textit{Ignoring utilizing intrinsic properties of utility function in federated learning}.
    {
        In SV-based data valuation within FL, we usually set the utility function to model accuracy, which is different from that in traditional game theory.
        For example, the utility function in \textit{weighted majority game} exhibits a binary jump, making it \#$\mathcal{P}$-hard~\cite{MOR94_SV} in this scenario.
        In contrast, the utility in FL (\eg model accuracy) typically exhibits a monotonic property as more clients join, providing it with unique features for data valuation.
    }

    To address these limitations, we propose an efficient and effective sampling-based approximation algorithm, \textit{Importance-Pruned Stratified Sampling (\textsl{IPSS})}. 
    Specifically, we first propose a unified stratified sampling framework, which can seamlessly integrate both the \textsl{\small MC-SV}-based and the \textsl{\small CC-SV}-based computation schemes.
    \rev{Then, we compare two computation schemes consistently in this framework under the assumption of FL linear regression model and then choose the \textsl{\small MC-SV} for further investigation.}
    We also identify a phenomenon in data valuation for FL through observations and empirical studies on utility function, referred to as \textit{key combinations}, which shows that different dataset combinations in FL have varying impacts on the final data value.
    Finally, we propose an approximation algorithm, \textsl{IPSS}, which prunes data combinations with minimal impacts from all possible combinations, significantly reducing time costs while maintaining accuracy.
    Main contributions of this work are summarized as follows.
    \begin{itemize}
        \item We propose a unified stratified sampling based framework that can integrate with both the \textsl{\small MC-SV}-based and the \textsl{\small CC-SV}-based computation schemes and then compare them to select the most promising one for further study.
        \item We identify a phenomenon called \textit{key combinations}, \ie data valuation in FL can be approximated accurately by utilizing only a small group of dataset combinations instead of all possible exponential combinations.
        \item We devise an efficient and effective algorithm, \textsl{IPSS}, tailored for the \textsl{\small MC-SV}-based data valuation in FL and then analyze its approximation error and time complexity.
        \item We conduct extensive experiments and compare our proposed \textsl{IPSS} algorithm with a series of baselines to validate its superiority in time cost and approximation error on both synthetic and benchmark datasets in FL.
\end{itemize}

In the rest of this paper, we first introduce the basic concepts of the data valuation in FL in \secref{sec:preliminary}.
Then, we propose a sampling-based framework and compare computation schemes in \secref{sec:framework}.
In \secref{sec:ps_alg}, we identify the key combinations phenomenon and introduce the proposed \textsl{IPSS} algorithm.
Finally, we present experimental evaluations in \secref{sec:exp}, review the related work in \secref{sec:related}, and conclude this work in \secref{sec:conclusion}.

\section{Problem Statement}\label{sec:preliminary}
In this section, we first present basic concepts of the \textit{federated learning} (FL) and the \textit{data valuation} problem in the context of FL.
Then, we introduce two equivalent computation schemes for the \textit{Shapley value} (SV) based data valuation.
\subsection{Preliminary and Problem Definition}

\begin{definition}[\textit{\textbf{Federated learning, FL}}]
    FL is a distributed learning paradigm that enables multiple data providers to utilize massive training data for the data-driven tasks without directly accessing raw data.\cite{yang2019federated, FL_survey_long, TKDE23_FL_Survey, CURS24_FL_Survey}.
    In FL, there are $n$ data owners (a.k.a. FL clients), each with a dataset $\mathcal{D}_i$, and a coordinator (a.k.a. FL server) and they aim to jointly train a learning model $M_{N}(\mathcal{A})$ across datasets from all clients through a FL algorithm $\mathcal{A}$, where $N=\{1,\dots, n\}$. 
\end{definition}

We take the most widely-used FL algorithm, \textit{FedAvg}~\cite{McMahanMRHA17}, as an example to illustrate the main process of FL.
A FL algorithm $\mathcal{A}$ operates iteratively at the FL server and clients as follows.
\textit{(i) Acts at server}:
In the first iteration, the FL server initializes and distributes the global model to all clients.
Otherwise, the FL server obtains a new global model by aggregating all local models in previous iteration from the clients.
\textit{(ii) Acts at clients}:
Take the client $i$ as an example. 
The client $i$ trains the received model on its local dataset $\mathcal{D}_{i}$ and then uploads an updated local model to the FL server.
The FL algorithm executes above two steps alternately until the required converge criterion or training round is reached.

    \begin{definition} [\textit{\textbf{Data valuation for FL}}] \label{def:dv4fl}
    Given $n$ datasets ${\mathcal{D}_{N}=\{\mathcal{D}_{1},\dots,\mathcal{D}_{n}\}}$ and a FL algorithm ${\mathcal{A}}$, the federation trains model ${M_{S}(\mathcal{A})}$ (or simply ${M_{S}}$) under a subset of clients $S\subseteq N$ and evaluates its utility as ${U(M_{S})}$ on the test dataset ${\mathcal{T}}$, where utility function ${U(\cdot)}$ is defined as model performance (e.g., accuracy).    
    Then, data valuation problem aims to quantify contribution of each dataset $\mathcal{D}_{i}$ as ${\phi(\mathcal{A},\mathcal{D}_{N},\mathcal{T}, \mathcal{D}_i)}$ (${\phi_{i}}$ for short), satisfying the following three basic properties,    
    \vspace{-0.5em}
    \end{definition}
    \begin{itemize}
    \item  \noindent\textit{(i) null-player (or no-free-riders)}: 
    If a dataset $\medop{\mathcal{D}_j}$ is irrelevant to FL model $\medop{M_{S}(\mathcal{A})}$ on test dataset $\medop{\mathcal{T}}$ for any dataset combination $\medop{\mathcal{D}_{S}}$, \rev{$\medop{\phi_{j}}$} should be zero.
    Formally,
    {
        \setlength{\abovedisplayskip}{5pt}
        \setlength{\belowdisplayskip}{5pt}
        \begin{equation}
           \medop
           {
           \forall S \subseteq N, U(M_S) = U(M_{S\cup \{j\}}) \Rightarrow \phi_j=0,
           }
        \end{equation}
    }
    \item   \textit{(ii) symmetric-fairness:}
    If two datasets $\medop{\mathcal{D}_i}$ and $\medop{\mathcal{D}_j}$ have the same effect on FL model $\medop{M_{S}(\mathcal{A})}$ on the test dataset $\medop{\mathcal{T}}$, their value in FL should be the same as well.
    Formally,
     {
        \setlength{\abovedisplayskip}{5pt}
        \setlength{\belowdisplayskip}{5pt}
        \begin{equation}
            \medop
            {
            \forall S\subseteq N\backslash\{i,j\}, U(M_{S\cup\{i\}})=U(M_{S\cup\{j\}}) \Rightarrow \phi_i=\phi_j,
            }
        \end{equation}
    }

    \item \textit{(iii) linear-additivity:} The data value for FL is linear with respect to two disjoint test dataset $\medop{\mathcal{T}_1}$ and $\medop{\mathcal{T}_2}$.    
    Formally, 
     {
        \setlength{\abovedisplayskip}{5pt}
        \setlength{\belowdisplayskip}{5pt}
        \begin{equation}
            \medop
            {
            \mathcal{T}_1 \cap \mathcal{T}_2 = \emptyset \Rightarrow \forall i\in N,  \phi_{i}(\mathcal{T}_1\cup{T}_2) = \phi_{i}(\mathcal{T}_1)+\phi_{i}(\mathcal{T}_2),
            }
        \end{equation}
    }
    where $\medop{\phi_{i}(\mathcal{T}_1)$, $\phi_{i}(\mathcal{T}_2)}$ and $\medop{\phi_{i}(\mathcal{T}_1\cup \mathcal{T}_2)}$ are assigned value to the dataset $\medop{\mathcal{D}_{i}}$ using same datasets $\medop{\mathcal{D}_{N}}$ and algorithm $\medop{\mathcal{A}}$.  
    \end{itemize}

    \fakeparagraph{Remarks}
    Above three properties are all essential in FL.
    Firstly, the \textit{null player}  is instrumental in identifying free riders who do not contribute to the FL model.
    The \textit{symmetric fairness} provides a fundamental fairness in FL, \ie if two datasets are interchangeable, they should have the same data value.
    Finally, the \textit{linear additivity} ensures that introducing new test data does not alter existing data value, \ie original data value remains reusable, simplifying the integration of new test data. 

    
\subsection{The Shapley Value based Data Valuation Schemes} \label{sec:preliminary_schemes}

    If we consider each FL client as a player and model performance as utility function in a collaborative game, then the \textit{Shapley value} (SV), a classical concept to fairly measure the player's contribution, naturally inherits its properties and ensures three desirable properties in \defref{def:dv4fl}\cite{Bigdata19_SV}.
    Therefore, the \textit{Shapley value} (SV) has been widely adopted as a standard data valuation metric in FL~\cite{VLDB24_Shapley, aistats19shapley, Bigdata19_SV, FLIP20_Wei, vldb19shapley, ICDE22_SV_Wang, VLDB23_SVF, Survey_SV_DB, Survey_SV_ML, TIST22_GTG, icml19shapley}.
    Furthermore, there are two commonly used equivalent SV expression, the marginal contribution based (\textsl{\small MC-SV}) and the complementary contribution based (\textsl{\small CC-SV}).
    Each provides a computation scheme for data valuation by its definition.

\begin{definition}[\textit{\textbf{\textsl{\small MC-SV} based computation scheme}}]
    Given $n$ datasets $\mathcal{D}_{N}=\{\mathcal{D}_{1}, \dots, \mathcal{D}_{n}\}$, a learning algorithm $\mathcal{A}$, test dataset $\mathcal{T}$ and the utility function $U(\cdot)$ in FL, \textsl{\small MC-SV} computes the data value for each dataset ${\phi(\mathcal{A}, \mathcal{T}, \mathcal{D}_{N}, \mathcal{D}_{i})}$ as follows,
    \begin{equation} \label{exp:DSV}
        \medop
        {
        \phi(\mathcal{A}, \mathcal{D}_{N}, \mathcal{T}, \mathcal{D}_i)=\sum_{S\subseteq N\backslash\{i\}}\frac{U(M_{S\cup\{i\}})-U(M_S)}{n\cdot\tbinom{n-1}{|S|}},
        }
    \end{equation}
    where $M_{S}$ denotes the FL model trained on dataset combination $\cup_{i\in S} \mathcal{D}_{i}$ and $|S|$ represents the number of datasets involved in $S$.
    The term $\tbinom{n-1}{|S|}=\frac{(n-1)!}{|S|!(n-1-|S|)!}$ is the combinatorial number.
    This computation scheme is referred to as the marginal contribution based SV (\textsl{\small MC-SV} for short), because it is based on the marginal contribution of each FL client ${i}$.
\end{definition}

    \begin{definition}[\textit{\textbf{\textsl{\small CC-SV} based computation scheme}}]
        Given datasets $\mathcal{D}_{N}$, learning algorithm $\mathcal{A}$, test datasets $\mathcal{T}$ and the utility function ${U}(\cdot)$ in FL, \textsl{\small CC-SV} computes the dataset $\mathcal{D}_{i}$'s data value ${\phi(\mathcal{A}, \mathcal{T}, \mathcal{D}_{N}, \mathcal{D}_{i})}$ in FL as follows,
        \begin{equation} \label{shapley_form:cc_shapley}
            \medop
            {
                \phi(\mathcal{A}, \mathcal{T}, \mathcal{D}_{N}, \mathcal{D}_{i}) = \sum_{S\subseteq N\backslash\{i\}}\frac{U(M_{S\cup\{i\}})-U(M_{N\backslash (S\cup\{i\})})}{n\cdot\tbinom{n-1}{|S|}},
            }
        \end{equation}
        where $\medop{U(M_{S\cup\{i\}})-U(M_{N\backslash (S\cup\{i\})})}$ is called the complementary contribution{\small ~\cite{SIGMOD23_SV_Zhang}} of clients $\medop{S\cup\{i\}}$ and we use \textsl{\small CC-SV} as the abbreviation for this computation scheme of Shapley value.
    \end{definition}
    \noindent\tabref{tlb:notions} summarizes the major symbols throughout this work.

\begin{table}[h]
\centering
{
    \setlength{\tabcolsep}{6pt}
    
    \caption{SV-based data valuation for FL with three clients.}
    \resizebox{0.75\linewidth}{!}{
        \centering\resizebox{\linewidth}{!}{
        \begin{tabular}{ccccccccc}
            \toprule
            $S$ & \( \emptyset \) & \( \{1\} \) & \( \{2\} \) & \( \{3\} \) & \( \{1, 2\} \) & \( \{1, 3\} \) & \( \{2, 3\} \) & \( \{1, 2, 3\} \) \\ \midrule
            $U(M_{S})$ & 0.10 & 0.50 & 0.70 & 0.60 & 0.80 & 0.90 & 0.90 & 0.96 \\ \bottomrule
        \end{tabular}
        }
    }
    \label{tlb:sv_example}
}
\vspace{-1em}
\end{table}

\begin{table}[htb]
    \centering
    \setlength{\tabcolsep}{5pt}
    \caption{Summary of the major symbol notions.}
    \label{tlb:notions}
    \resizebox{0.85\linewidth}{!}{
    \begin{tabular}{cl}
    \toprule
    \textbf{Notations}          & \textbf{\qquad\qquad\qquad\qquad Descriptions}         \\
    \toprule
    $N, S$  & the set of all FL clients and a subset of all clients \\ 
    $\phi_{i}, \hat{\phi}_{i}$ &  FL client $i$'s data value and its approximated value  \\
    $\mathcal{D}_{S}, \mathcal{D}_{i}$ & datasets of client combination $S$ and of FL client $i$ \\
    $\mathcal{A}, \mathcal{T}$ & the training algorithm and the test dataset in FL \\
    $M_{S}$  & FL model trained over datasets held by clients $S$ \\
    $U(\cdot)$  & the utility function in SV-based data valuation \\
    \textsl{\footnotesize MC-SV}  &the SV scheme based on marginal contribution \\
    \textsl{\footnotesize CC-SV}  &the SV scheme based on complementary contribution \\
    $\mathcal{S}_{k}$ & dataset combinations with datasets of $k$ clients \\
    $|S|$ & the number of FL clients involved in combination $S$\\
    $\gamma$ &  total sampling rounds in approximation algorithm \\
    $\tau$ & the time cost for training and testing a FL model \\
    \bottomrule
    \end{tabular}
    }
    \vspace{-1.5em}
\end{table}


\begin{example}  \label{ex:mcsv}
    Considering a FL scenario with three clients $\medop{N=\{1, 2, 3\}}$.
    The utilities of FL models for each possible client combination are detailed in \tabref{tlb:sv_example}.
    We take FL client $1$ as an example whose data value is denoted as $\phi_1$.
    In this example, we employ the \textsl{\small MC-SV}-based computation scheme.
    The data value $\phi_{1}$ is determined by averaging the marginal contribution of FL client $1$ when added to all combinations without it.
    \begin{itemize}
        \item  For the empty combination $\medop{\emptyset}$, the marginal contribution of adding the FL client $\medop{1}$ is $\medop{U(\{1\})-U(\emptyset) = 0.40}$.
        \item For combinations including one other FL client, the marginal contributions of adding the client $\medop{1}$ are $\medop{U(\{1,2\})-U(\{2\}) = 0.10}$ and $\medop{U(\{1,3\})-U(\{3\}) = 0.30}$.
        \item  For the combination with other two FL clients, the marginal contribution is $\medop{U(\{1,2,3\})-U(\{2,3\}) = 0.06}$.
    \end{itemize}
    Finally, these contributions are averaged to compute $\phi_{1}$ as
    $\medop{({0.40}/{1}+{(0.10+0.30)}/{2}+{0.06}/{1})/3=0.22}$.
    Similarly, data value of clients $\medop{2}$ and $\medop{3}$ are $\medop{\phi_{2}\approx 0.32}$ and $\medop{\phi_{3}= 0.32}$.
\end{example}

\subsection{Approximations for SV-based Data Valuation in FL}

Though Shapley value has been widely adopted as a standard metric \cite{VLDB24_Shapley, Bigdata19_SV, FLIP20_Wei, vldb19shapley, ICDE22_SV_Wang, VLDB23_SVF, FLIP20_FSV, TIST22_GTG}, it is acknowledged that \textit{the high computational overhead  of Shapley value  prohibits its practical use, as it requires evaluating for all possible dataset combinations.}
Specifically, using both \textsl{\small MC-SV} and \textsl{\small CC-SV} based schemes needs to train and assess ${\mathcal{O}(2^n)}$ FL models, which is infeasible in practice.
Thus, it is imperative to devise practical approximation methods for the SV-based data valuation in FL that meet the following two critical two requirements.
\begin{itemize}
    \item \textit{R1: It is necessary to efficiently estimate SV-based data value for FL dataset.} 
    In commercial applications, FL clients may often prioritize their computational resources for model training and deployment.
    Excessive time cost spent on data valuation in FL should be avoided to ensure the economic interests in real-world scenarios.
    
    \item \textit{R2: It is crucial to approximate SV-based data value for FL clients with tolerable errors.}
    The Shapley value is favored for its desirable fairness properties (\eg \textit{no free riders and symmetric fairness}).
    However, substantial approximation errors can undermine these fairness properties, jeopardizing its applicability for data providers.
\end{itemize}

\fakeparagraph{Roadmap for Approximations}
    To address the two requirements (\textit{efficiency} and \textit{effectiveness}), existing literature ~\cite{Bigdata19_SV, FLIP20_Wei, vldb19shapley, FLIP20_FSV, TIST22_GTG, ICDE22_SV_Wang, VLDB23_SVF} primarily explore two types of approximation methods: \textit{gradient-based approximation} and \textit{sampling-based approximation}.
    \textit{(i)} The \textit{gradient-based} approaches~\cite{Bigdata19_SV, FLIP20_Wei,TIST22_GTG} utilize the gradients during the FL training process to construct the FL models under various dataset combinations, which avoids training FL models from scratch, thereby reducing the computational time significantly.
    However, these gradient-based methods usually lack theoretical underpinnings, which can result in higher approximation errors.
    \textit{(ii)} The \textit{sampling-based} solutions~\cite{vldb19shapley, icml19shapley, AISTATS23_DV} strategically select subsets from entire potential combinations, providing a flexible trade-off between accuracy and efficiency, making it more suited to meet the efficiency and effectiveness requirements.
    Therefore, this paper advocates for the \textit{sampling-based} approximations and we propose a novel sampling-based approach for practical SV-based data valuation in FL, which is detailed in the following sections (in \secref{sec:framework} and \secref{sec:improve}).

\section{Stratified Sampling based Framework} \label{sec:framework}
    In this section, we first introduce a unified stratified sampling framework that integrates both computation schemes outlined in \secref{sec:preliminary_schemes}.
    Then, we analyze and choose the \textsl{\small MC-SV}-based computation scheme as the more appropriate choice for the proposed stratified sampling-based framework.
        
\subsection{The Unified Sampling Framework for SV-based Schemes}
    By definition, both the \textsl{\small MC-SV} and \textsl{\small CC-SV} possess an inherent hierarchical structure based on the size of dataset combinations (as shown in \figref{fig:example_comb_shapley}), where each FL client’s data value is calculated by the average marginal (or complementary) contributions across combinations of various sizes. 
    Thus, it is natural to treat dataset combinations of the same size as strata and employ the stratified sampling for approximation. 
    To this end, we devise a unified stratified sampling framework to support both \textsl{\small MC-SV}-based and \textsl{\small CC-SV}-based computation schemes, which is illustrated in \algref{alg:mc4sv}.
    Then we can compare these two schemes in a consistent framework.
    
    \fakeparagraph{Main Idea}  
    Let $\mathcal{S}_{k}$ denotes all dataset combinations with datasets from $k$ clients and we can use the {Monte Carlo} method to approximate a stratified-SV $\hat{\phi}_{i,k}$ for dataset combination $\mathcal{S}_{k}$ with $k$ datasets (\ie each stratum) and the estimated SV $\hat{\phi}_{i}$ is the average across all strata.
    Specifically, \algref{alg:mc4sv} takes as input sampling rounds $m_{k}$ for the $k_{\text{th}}$ stratification ($\gamma = \sum_{k=1}^{n} m_{k}$), along with $n$ datasets $\mathcal{D}_{1}, \dots, \mathcal{D}_{n}$ and a utility function to measure the performance of trained FL model.
    In lines 1-8, the framework first randomly samples dataset combinations used for each FL client and tests the utility of FL models for each stratum.
    Then, the framework calculates the stratified-SV for each combination size and it can be easily integrated with both the \textsl{\small MC-SV}-based and the \textsl{\small CC-SV}-based data valuation computation schemes.
    For \textsl{\small MC-SV}, two combinations $S$ and $\overline{S}$ are paired when $\overline{S}={S}\backslash\{i\}$, whereas in \textsl{\small CC-SV}, $S$ are paired with $\overline{S}=N\backslash S$.
    Finally, the framework approximates and returns the SV by averaging marginal (or complementary) contributions within each stratum. 
    
        \begin{figure}[tb]
        \centering
        \includegraphics[width=0.75\linewidth]{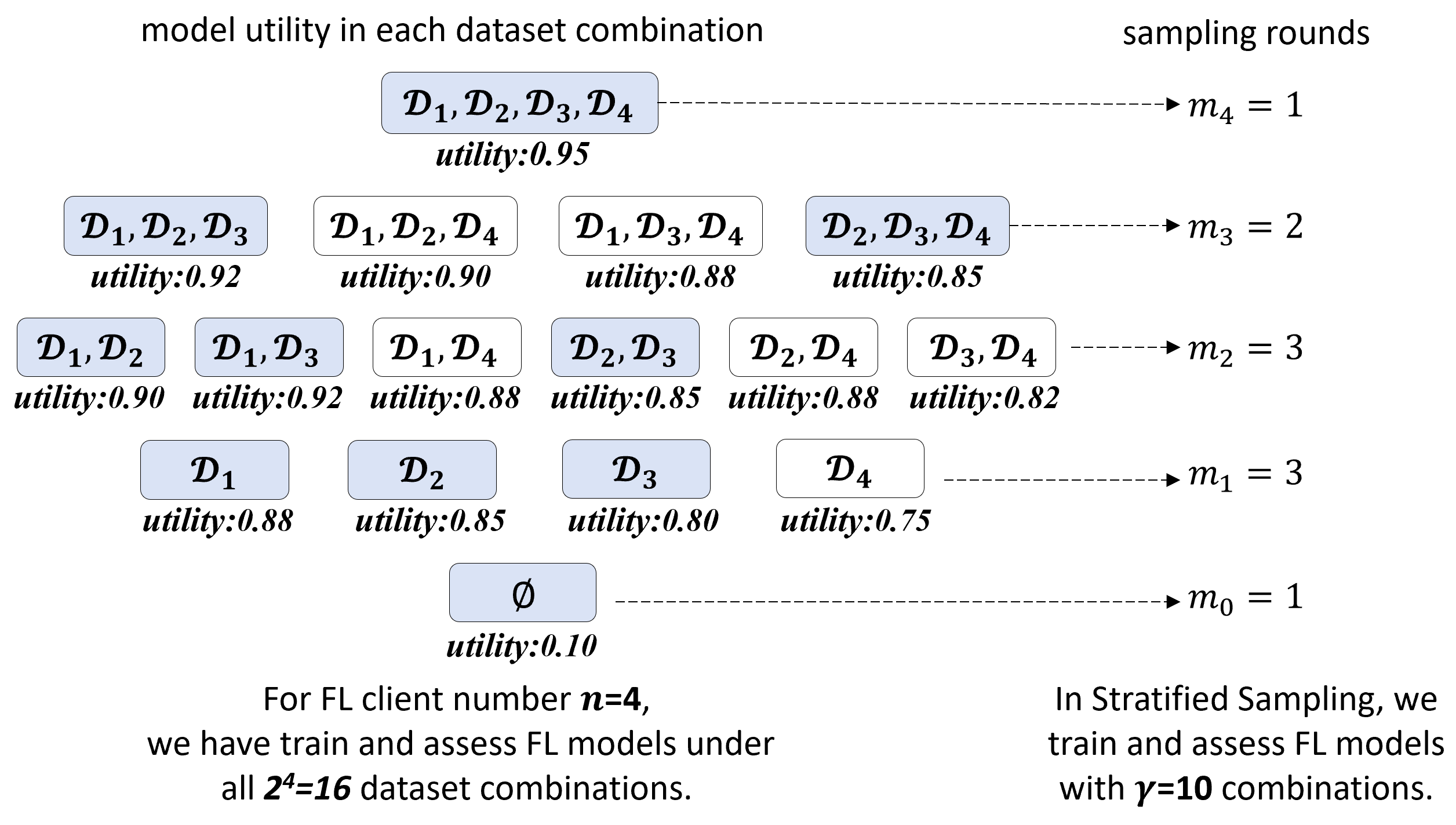}
        \caption{
            Example for the unified stratified sampling framework:
            {
               \small 
               Both \textsl{\small MC-SV} and \textsl{\small CC-SV} rely on this hierarchical structure, which is naturally suitable for stratified sampling.              
                There are four FL clients and model utility is below each dataset combination.
                For instance, the utility of FL model under dataset combination $\medop{\{\mathcal{D}_1, \mathcal{D}_3\}}$ is $\medop{0.92}$.
            }
        }
        \label{fig:example_comb_shapley}
    \end{figure}
    
    \begin{algorithm}[htb]
        \caption{Stratified Sampling Approximation}
        \label{alg:mc4sv}
        \SetAlgoLined
        \KwIn{The $n$ data providers with datasets $\mathcal{D}_{N}$, a test dataset $\mathcal{T}$ and a utility function $U(\cdot)$.
        Sampling rounds for each stratum $m_{k}$ ($\gamma=\sum_{k}m_{k}$).} 
        \KwOut{Data value of all datasets $\hat{\phi_1}, \cdots, \hat{\phi_n}$.}
        
        \For{$k \leftarrow 1 \ \mathbf{\mathrm{to}}\  n $}{
                Initialize $\mathcal{S}_{k}$ as all combinations with $k$ datasets\;
                $\textit{\textbf{S}}_{k} \leftarrow \{S_{k,1}, \dots, S_{k,m_{k}}\}$ w.r.t. $S_{k, \cdot}  \sim  \mathcal{S}_{k}$\;
                $\textit{\textbf{S}}_{k,i} \leftarrow \{S\vert S\in \mathcal{S}_{k} \text{ and } i\in S\}$\;
                
                \For{$S \in \textbf{S}_{k}$}{
                    Train and evaluate FL model $M_{S}$ and then we can obtain the model's utility $U(M_{S})$ on $\mathcal{T}$\;
                    
                }
        }
        \For{$i \leftarrow  1 \ \mathbf{\mathrm{to}}\  n$}{
            \For{$S \in \textbf{S}_{k,i}$}{
                \If{ the paired combination $\overline{S}$ is sampled}{
                    $m_{i,k} \leftarrow m_{i,k} + 1$\;
                    $\hat{\phi}_{i,k}\leftarrow \medop{\hat{\phi}_{i,k} + U(M_{S})-U(M_{\overline{S}})}$\;
                }
            }
        }
                $\hat{\phi}_{i} \leftarrow \frac{1}{n}\sum_{k=1}^{n}\frac{\hat{\phi}_{i,k}}{m_{i,k}}\   \medop{(\text{where}\ i=1,2,\dots,n)}$\;
        \KwRet $\hat{\phi}_1, \cdots, \hat{\phi}_n$
    \end{algorithm}
    
    \begin{example} \label{ex:example_ss_framework}
    We illustrate the stratified sampling framework by the example in \figref{fig:example_comb_shapley}.
    We set the total sampling round $\medop{\gamma=10}$ and sampled dataset combinations are marked in light blue.

    \noindent \textit{\underline{Case 1 (using \textsl{\small MC-SV}):}}
    Set the adopted computation scheme in \algref{alg:mc4sv} to \textsl{\small MC-SV}, we calculate the average marginal contribution for each stratum in this case.
    We take the FL client $1$'s data value as an instance which is $\medop{0.26}$ by its definition.
    For dataset combination $\medop{\mathcal{S}_{1}}$, we calculate marginal contribution of FL client $1$ as $\medop{\hat{\phi}_{1,1}=U(M_{\{1\}})-U(M_{\emptyset})=0.78}$.
    For combinations with two datasets $\medop{\mathcal{S}_{2}}$, we can calculate the average marginal contribution in this stratum as $\medop{\hat{\phi}_{1,2}=(U(M_{\{1,2\}})-U(M_{\{2\}})+U(M_{\{1,3\}})-U(M_{\{3\}}))/2=0.085}$. 
    Similarly, for $\medop{\mathcal{S}_{3}}$ and $\medop{\mathcal{S}_{4}}$, we have $\medop{\hat{\phi}_{1,3}=0.07}$ and $\medop{\hat{\phi}_{1,4}=0.10}$.
    Finally, the data value of FL client $\medop{1}$ can be approximated by $\medop{\hat\phi_1=(0.78+0.085+0.07+0.10)/4\approx 0.2588}$.

    \noindent \textit{\underline{Case 2 (using \textsl{\small CC-SV}):}}
    We take the average complementary contribution for each stratum in this case.
    For dataset combination with one dataset $\medop{\mathcal{S}_1}$, the complementary contribution of FL client $\medop{1}$ is $\medop{\hat\phi_{1,1}=U(M_{\{1\}})-U(M_{\{2,3,4\}})=0.03}$.
    For dataset combinations with two datasets $\medop{\mathcal{S}_2}$, the complementary contribution is calculated as $\medop{\hat\phi_{1,2}=0}$, since no paired combination $\medop{N\backslash S_2 \ (S_{2}\in \mathcal{S}_{2})}$ is sampled.
    Similarly, we have $\medop{\hat\phi_{1,3}}=0$ and $\medop{\hat\phi_{1,4}=0.85}$ for $\medop{\mathcal{S}_3}$ and $\medop{\mathcal{S}_4}$.
    Finally, the FL client $1$'s data value is approximated by $\medop{\hat\phi_1=(0.03+0+0+0.85)/4=0.22}$.

        
    \end{example}
\subsection{Choosing Computation Scheme for Sampling Framework} \label{sec:framework_comp_scheme}


    As mentioned above, the \algref{alg:mc4sv} supports both the \textsl{\small MC-SV}-based and the \textsl{\small CC-SV}-based computation schemes, allowing us to analyze and compare them within a consistent sampling based framework.
    Since \textsl{\small MC-SV}-based and \textsl{\small CC-SV}-based schemes have the same time complexity of $\mathcal{O}(2^n \tau)$ based on the definition, where $\tau$ denotes time cost of training and assessing a FL model.
    We further compare their performance in expectation and variance using a consistent sampling strategy.
   

    \begin{theorem}
        The \algref{alg:mc4sv} provides an unbiased estimation of SV in expectation when using both the \textsl{\small MC-SV} or the \textsl{\small CC-SV}.  
    \end{theorem}
    \begin{proof} 
        We first analyze the expectation of stratified-SV $\hat{\phi}_{i,k}$,
        \begin{equation} 
            \begin{aligned}
             &\medop{
                \mathbb{E}[\frac{\hat{\phi}_{i,k}}{m_{i,k}}] = \frac{\mathbb{E}_{S\sim \mathcal{S}_{k,i}}[\hat{\phi}_{i,k}]}{m_{i,k}} = \frac{\mathbb{E}_{S\sim \mathcal{S}_{k,i}}
                [\sum_{t=1}^{m_{i,k}}{U(M_{S})-U(M_{\overline{S}})}]}{m_{i,k}}
             }\\
             &\medop{
                = \mathbb{E}_{S\sim \mathcal{S}_{k,i}}
                [U(M_{S})-U(M_{\overline{S}})] = \sum_{S\subseteq N\backslash\{i\}}\frac{U(M_{S\cup\{i\}})-U(M_{\overline{S\cup\{i\}}})}{\tbinom{n-1}{|S|}}
            } 
            \end{aligned}
        \end{equation}
        Then, the expectation of the SV based on \algref{alg:mc4sv} is,
        \begin{equation} \label{eq:SV_E}
            \begin{aligned}
                \medop{
                    \mathbb{E}[\hat{\phi_{i}}] = \mathbb{E}[\frac{1}{n}\cdot\sum_{k=1}^{n}{\hat{\phi}_{i,k}}]
                }
                \medop{
                    =\frac{1}{n}\sum_{S\subseteq N\backslash\{i\}}\frac{U(M_{S\cup\{i\}})-U(M_{\overline{S\cup\{i\}}})}{\tbinom{n-1}{|S|}}
                }
            \end{aligned}
        \end{equation}
        Thus, by the definition of \textsl{\small MC-SV} and \textsl{\small CC-SV}, \equref{eq:SV_E} equals the FL client $i$'s data value $\phi_{i}$, completing our proof.
    \end{proof}

    
    \begin{theorem} \label{thm:CC_Variance}
        Assume that the data from all providers are all drawn from the same distribution and let $|\mathcal{D}_i|$ denote the size of dataset held by the FL client $i$. 
        Then for any sampling strategy for \textsl{\small CC-SV} based scheme, using the \textsl{\small MC-SV} based scheme can yield a lower variance in \algref{alg:mc4sv} in FL linear regression.
    \end{theorem}

    \begin{proof}
        Based on the theoretical analysis in \cite{Book15_LR_Model}, the variance of error in a linear regression model applied to a dataset $\mathcal{D}$ with $t$ training samples, can be described as follows, 
        \begin{equation}
        \begin{aligned}            
           \medop
           {\mathbb{V}ar[U(M_{\mathcal{D}})] = \mathbb{V}ar[\sum_{j=1}^{t}e_j]=\sum_{j=1}^{t}\mathbb{V}ar[\vert \hat{f}(x_j)-y_j\vert] =  t^2\sigma^2}
        \end{aligned}
        \end{equation}
    where $e_j=\vert \hat{f}(x_j)-y_j\vert$ is the mean absolute error on each training sample $(x_j, y_j)$ and $\sigma^2$ is the variance of intrinsic random noise in the dataset.
    If we take negative mean average error as utility and randomly sample a dataset combination $S$ from $\medop{N\backslash\{i\}}$ to approximate \textsl{\small MC-SV}, its variance $\medop{\mathbb{V}ar[\hat{\phi}_{i}^{MC}]}$ is,
    \begin{equation} \label{eq:sv_var_comb}
    \begin{aligned}
        &\medop{\mathbb{V}ar[\hat{\phi}_{i}^{MC}]
        = 
        \mathbb{V}ar[\frac{1}{n}\sum_{k=1}^{n}\sum_{S\sim (N\backslash\{i\})}\frac{U(M_{S\cup\{i\}})-U(M_{S})}{ m_{i,k}}]
        }   \\
        &\medop{
        =\sum_{k=1}^{n} \sum_{S} \frac{\mathbb{V}ar[U(M_{S\cup\{i\}})-U(M_{S})]}{n^2\cdot m_{i,k}^2}} \\
        &\medop{=\sum_{k=1}^{n} \sum_{S}\frac{1}{n^2 m_{i,k}^2}\cdot \mathbb{V}ar[-\sum_{j\in \mathcal{D}_{S\cup\{i\}}} e_j + \sum_{j\in \mathcal{D}_{S}} e_j]} \\
        &\medop{=\sum_{k=1}^{n} \sum_{S}\frac{1}{{n^2  m_{i,k}^2}}\cdot \mathbb{V}ar[\sum_{j\in \mathcal{D}_{i}} e_j]=\sum_{k=1}^{n}\sum_{S}\frac{1}{n^2 m_{i,k}^2}  |\mathcal{D}_i|^2\sigma^2}
    \end{aligned}
    \end{equation}
    Similarly, the variance of the \textsl{\small CC-SV} can be calculated as,
    \begin{equation}  \label{eq:sv_var_cc}
    \begin{aligned}
        &\medop{
        \mathbb{V}ar[\hat{\phi}_{i}^{CC}] = \mathbb{V}ar[\frac{1}{n}\sum_{k=1}^{n}\sum_{S\sim (N\backslash\{i\})}\frac{U(M_{{S}\cup \{i\}})-U(M_{N\backslash (S\cup\{i\})})}{m_{i,k}}]} \\
        &\medop{=\sum_{k=1}^{n}\sum_{S} \frac{1}{n^2  m_{i,k}^2} \cdot (\mathbb{V}ar[\sum_{j\in \mathcal{D}_{S\cup\{i\}} }e_j]+\mathbb{V}ar[\sum_{j\in \mathcal{D}_{N\backslash(S\cup\{i\})}}e_j])} \\
        &\medop{=\sum_{k=1}^{n}\sum_{S} \frac{1}{n^2  m_{i,k}^2} ((|\mathcal{D}_S|+|D_i|)^2+(|\mathcal{D}_N|-|\mathcal{D}_S|-|\mathcal{D}_i|)^2)\sigma^2}
    \end{aligned}
    \end{equation}
    If they take the same sampling strategy in approximation, we can compare their variances by \eqref{eq:sv_var_cc} - \eqref{eq:sv_var_comb} as follows,
    \begin{equation}
        \medop{
        \mathbb{V}ar[\hat{\phi}_{i}^{CC}] - \mathbb{V}ar[\hat{\phi}_{i}^{MC}] \ge \sum_{k=1}^{n}\sum_{S} \frac{1}{n^2  m_{i,k}^2} |\mathcal{D}_S|^2\sigma^2 > 0 
        }
    \end{equation}
    Therefore, the variance of \textsl{\small MC-SV} is lower than that of \textsl{\small CC-SV} when using \algref{alg:mc4sv} and we finish the proof of \thmref{thm:CC_Variance}.
    \end{proof}
    
    \fakeparagraph{Takeaways}
    Based on above theoretical analysis, we have two main results: \textit{(i) when implemented within \algref{alg:mc4sv}, \textsl{\small MC-SV}-based and \textsl{\small CC-SV}-based computation schemes can both provide unbiased estimations for data valuation.} 
    \textit{(ii) for each stratified sampling strategy of \textsl{\small CC-SV}, there exists a corresponding strategy of \textsl{\small MC-SV} that yields lower variance in the context of FL linear model.}
    It is important to note that \algref{alg:mc4sv} operates as a stratified sampling framework without imposing specific assumptions on the number of sampling rounds $m_i$ for each stratum.
    Thus, our analysis provides broadly applicable evidence when comparing the two schemes and justifying our choice of the \textsl{\small MC-SV} based computation scheme.

\section{Importance-Pruned Stratified Sampling} \label{sec:improve}
\label{sec:ps_alg}
    In this section, we explore the \textsl{\small MC-SV}-scheme within our stratified sampling framework to devise a practical approximation algorithm.
    Initially, we further observe the \textsl{\small MC-SV}-based computation scheme and identify a phenomenon, called \textit{key combinations}, \ie \textit{it is sufficient to focus on only a selected subset of all $2^{n}$ dataset combinations to approximate the SV in FL.}
    Armed with this knowledge, we can selectively prune the less significant dataset combinations to optimize efficiency and accuracy for the \textsl{\small MC-SV}-based data valuation in FL.

    \begin{figure}[thb]
        \centering
        \includegraphics[width=0.65\linewidth]{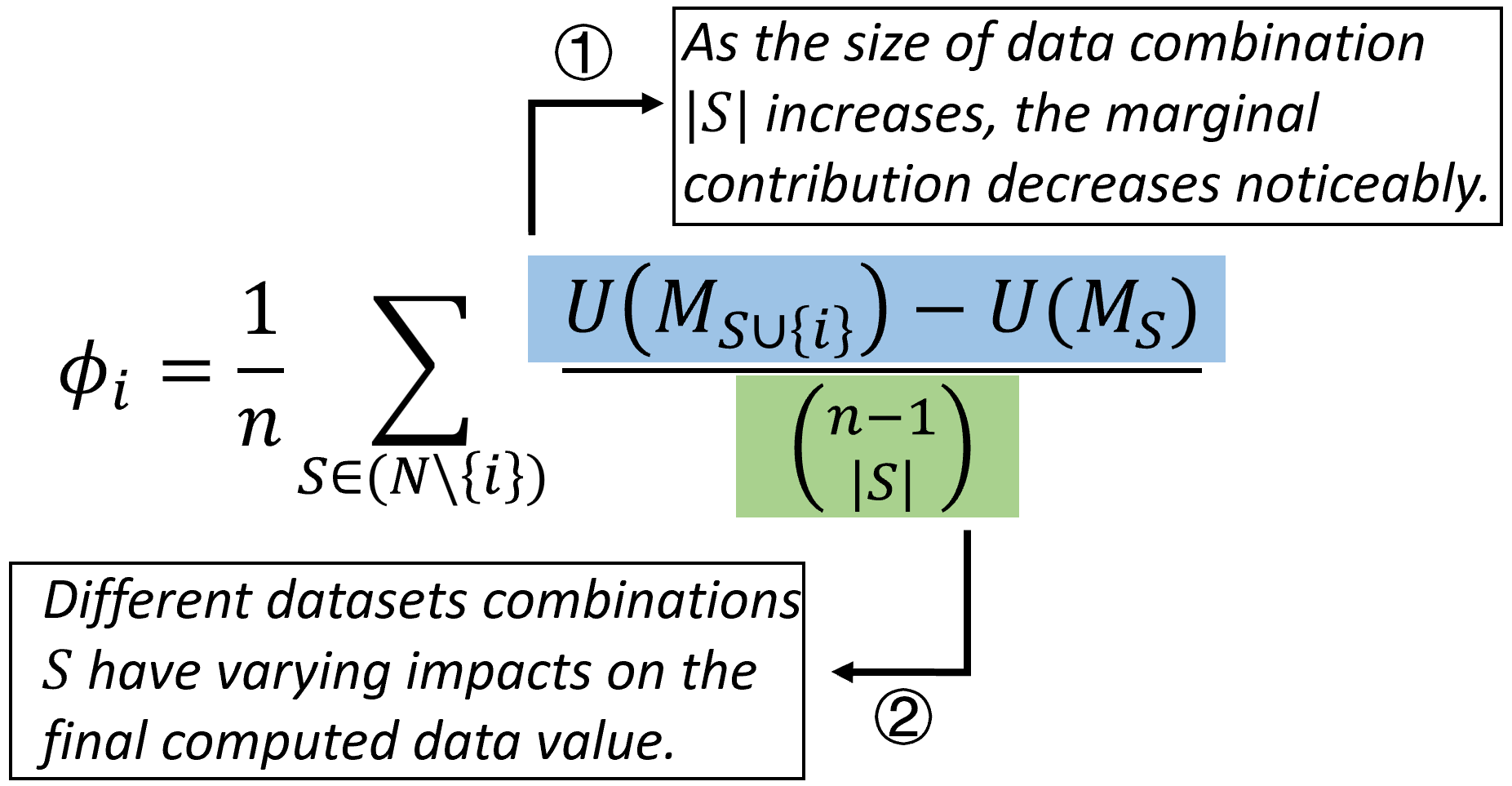}
        \caption{Observations when using the \textsl{\small MC-SV}-based scheme.}
        \label{fig:intuition_comb_shapley}
        \vspace{-1.5em}
    \end{figure}
    
\subsection{Identifying the \textit{Key Combinations} Phenomenon} \label{sec:improve_KCP}


    \fakeparagraph{Observations} 
    It is essential to utilize the inherent properties of the \textsl{\small MC-SV} for effective and efficient data valuation in FL.
    To this end, we recall and examine the \textsl{\small MC-SV}-based computation scheme as depicted in \figref{fig:intuition_comb_shapley}.
    This analysis reveals that different dataset combinations $S$ have varying impacts on the data value $\phi_{i}$ primarily in two aspects:
    \begin{itemize}
        \item \textit{(i)} The marginal utility of FL model decreases with the addition of more datasets.
        Once there are already sufficient datasets for training in FL, adding a new dataset can only improves limited utility, \ie when $\mathcal{D}_{S}$ are large, the marginal utility $\medop{U(M_{S\cup\{i\}})-U(M_{S})}$ is usually small.        
        \item \textit{(ii)} For certain dataset combination $S$, if its size $|S|$ is close to $\medop{(n-1)/{2}}$, its impact on the final result tends to be limited as well, because its coefficient $\medop{1/\tbinom{n-1}{|S|}}$ in \textsl{\small MC-SV} is much smaller compared to others. 
    \end{itemize}
    The above observations \textit{(i)} and \textit{(ii)} indicate that the impact on the estimated \textsl{\small MC-SV} diminishes when the size of the dataset subset $|S|$ approaches or exceeds $(n-1)/2$.
    This leads us to conjecture that \textit{the contributions of datasets in FL are predominantly influenced by a select few combinations of datasets $\medop{S \subseteq N}$, particularly those involving smaller FL clients}.
    To test and validate these observations, we further develop a simple algorithm called $K$-Greedy (in \algref{alg:kgreedy}), which adopts the \textsl{\small MC-SV}-based scheme.
    \algref{alg:kgreedy} only focuses on combinations with no more than $K$ datasets, intentionally disregarding impacts of combinations with more FL clients.
    
    \fakeparagraph{Empirical Setups}
    To validate our conjecture above, we embark on an empirical investigation of the $K$-Greedy algorithm.
    Our experiments take the FL benchmark dataset FEMNIST \cite{arXiv18_LEAF} and employ the widely-used convolutional neural network as the FL model.
    Without loss of generality, we partition the FEMNIST dataset into 10 distinct data providers (\ie FL clients), each holds digits contributed by different writers.
    
    \begin{figure}[htb]
        \centering
        \includegraphics[width=0.7\linewidth]{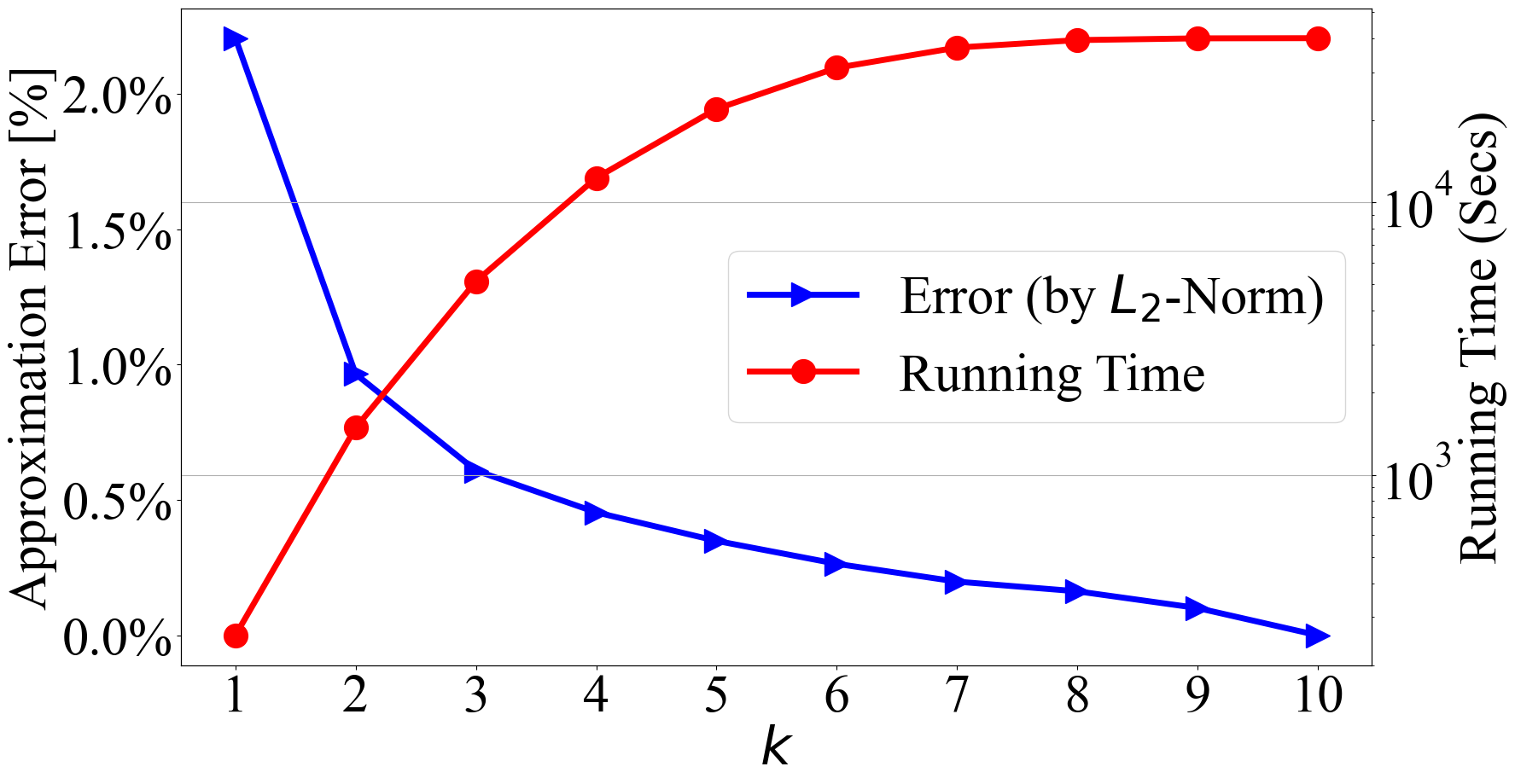}
        \caption{Results under combinations with size no more than $K$.}
        \label{fig:kg_shap}
        \vspace{-0.5em}
    \end{figure}
    
    \begin{algorithm}[thb]
        \caption{$K$-Greedy}
        \label{alg:kgreedy}
        \SetAlgoLined
        \KwIn{The $n$ datasets $\mathcal{D}_{N}$, a test dataset $\mathcal{T}$, a utility function $U(\cdot)$ in FL, and a constant number $K$.} 
        \KwOut{The estimated data value $\hat{\phi}_1, \dots, \hat{\phi}_n$}
        // Evaluate the utility of combination of datasets. \\
        \For{$S \subseteq N  \ \mathbf{\mathrm{and}}\   |S| \leq K$}{
                Train the model $M$ and evaluate its utility $U(M_{S})$\;
        }
        // Approximate the \textsl{MC-SV} for each data providers. \\
        \For {$i \in N$}{
        $\hat{\phi}_{i} \leftarrow \sum_{S\subseteq (N\backslash\{i\}), |S| < K}\frac{U(M_{S\cup\{i\}})-U(M_{S})}{n\cdot\tbinom{n}{|S|}}$ 
        \;
        }
        \KwRet $\hat{\phi}_1, \cdots, \hat{\phi}_n$
    \end{algorithm}
    
     \begin{algorithm}[htb]
        \caption{Importance-Pruned Stratified Sampling}
        \label{alg:lightSampling}
        \SetAlgoLined
        \KwIn{The $n$ datasets $\{D_1,\cdots,D_n\}$, a test dataset $\mathcal{T}$, a utility function $U(\cdot)$ and sampling rounds $\gamma$} 
        \KwOut{Data value for $n$ datasets $\hat{\phi}_1, \cdots, \hat{\phi}_n$}
        $k^* \leftarrow \max \{ k\in \mathbb{N} | \sum_{j=0}^{k} \tbinom{n}{j} \le \gamma $\}\;
        \For{$ k \leq k^*$}{
            Initialize $\mathcal{S}_{k}$ with all combinations with $k$ datasets\;

            \For{$S \in \mathcal{S}_{k}$}{
                    Train and evaluate the FL model $M_{S}$ on dataset combination $\mathcal{D}_{S}$ and test dataset $\mathcal{T}$\;
            }
        }
        Sampling a set of dataset combinations $\mathcal{P}$ such that:\\
            \ {\small \textit{(1)}} $|\mathcal{P}| \leq \gamma- \sum_{j=0}^{k^*} \tbinom{n}{j}$ \;
            \ {\small \textit{(2)}}  $\forall S\in \mathcal{P}, |S|=k^*+1$\;
            \ {\small \textit{(3)}}  $\forall i,j\in N, C_i=C_j$ where $C_k = \sum_{S\in \mathcal{P}}\mathbb{I}[k \in S]$ \;
          \For{$S \in \mathcal{P}$}{
                   Train and evaluate the FL model $M_{S}$ on dataset combination $\mathcal{D}_{S}$ and test dataset $\mathcal{T}$\;
            }
            
       \For{$i \leftarrow  1 \ \mathbf{\mathrm{to}}\  n$}{
        $\hat{\phi}_{i} \leftarrow \frac{1}{n}\sum_{S\subseteq (N\backslash\{i\}), |S| < k^*}\frac{U(M_{S\cup\{i\}})-U(M_{S})}{\tbinom{n-1}{|S|}}$ 
        $+ \frac{1}{n}\sum_{S\subseteq (N\backslash\{i\}), (S\cup\{i\})\in \mathcal{P}}\frac{U(M_{S\cup\{i\}})-U(M_{S})}{\tbinom{n-1}{|S|}}$\;
            
        }
        \KwRet $\hat{\phi}_1, \cdots, \hat{\phi}_n$
    \end{algorithm}
    
    \fakeparagraph{Key Combinations Phenomenon} 
    The empirical results are shown in \figref{fig:kg_shap}.
    To quantify the approximation error, we adopt the relative error metric, defined as $\frac{\Vert \phi- \hat{\phi} \Vert_{2}}{\Vert \phi \Vert_{2}}$, where $\phi$ represents the data value calculated by \textsl{\small MC-SV}, and $\hat{\phi}$ denotes the approximated data value in FL.
    The empirical results show that for dataset combinations of size ${K \leq 2}$, the relative error is less than $1\%$, which suggests that using dataset combinations involving no more than 2 FL clients allows for a highly accurate approximation of the \textsl{\small MC-SV}.
    Besides, the relative error decreases rapidly as $K$ increases from 1 to 3 and the rate of decrease becomes more gradual as $K$ becomes larger, which implies that dataset combinations with a larger number of clients have less impact on the final \textsl{\small MC-SV} in FL, which aligns with our conjecture.
    To summarize, we characterize this observed phenomenon as the \textit{key combinations}, \textit{where a limited number of dataset combinations, typically involving a few clients, play a pivotal role when we take the \textsl{\small MC-SV}-based computation scheme for data valuation in FL}.


\subsection{Importance-Pruning for Acceleration}


    \fakeparagraph{IPSS Algorithm}
    Building upon the above empirical evidence, we further refine the proposed stratified sampling framework as described in \secref{sec:framework}.
    We introduce a novel approximation algorithm, importance-pruned stratified sampling (\textsl{IPSS}), tailored for \textsl{\small MC-SV}-based data valuation in FL.
    Given the total sampling rounds $\gamma$, the \textsl{IPSS} prunes dataset combinations involving a large number of FL clients, focusing only on those combinations that have a high-impact on final results.
    The \textsl{IPSS} is detailed in \algref{alg:lightSampling}.
    Given $n$ datasets of FL clients, the utility function $\medop{U(\cdot)}$ and the sampling rounds $\medop{\gamma}$, the algorithm is design to efficiently approximate contributions of these datasets in FL.
    The \textsl{IPSS} algorithm operates in two phases.
    \textit{(i) Initially, the algorithm evaluates the utility of FL model trained on various dataset combinations.}
    In lines 1-7, the \textsl{IPSS} calculates the maximum size of used dataset combinations $k^*$, and then we train and evaluate the FL model on dataset combinations whose sizes do not exceed $k^*$.
    In lines 8-11, for remaining sampling rounds, the \textsl{IPSS} samples dataset combinations of size $k^*+1$ and ensures equal sampling frequency for each dataset, thereby providing a fair approximation error across FL clients. 
    In lines 12-14, we train and evaluate the utility of FL models under these dataset combinations containing $k^*+1$ datasets.
    \textit{(ii) The \algref{alg:lightSampling} approximates data value based on \textsl{\small MC-SV}.}
    In lines 15-17, the algorithm takes the evaluated dataset combinations as proxies for all combinations and estimates the data value by the \textsl{\small MC-SV}-based computation scheme, reducing the computational overhead.
    Finally, the \textsl{IPSS} returns data value for each dataset.

    \begin{figure}[htb]
        \centering
        \vspace{-1em}
        \includegraphics[width=0.75\linewidth]{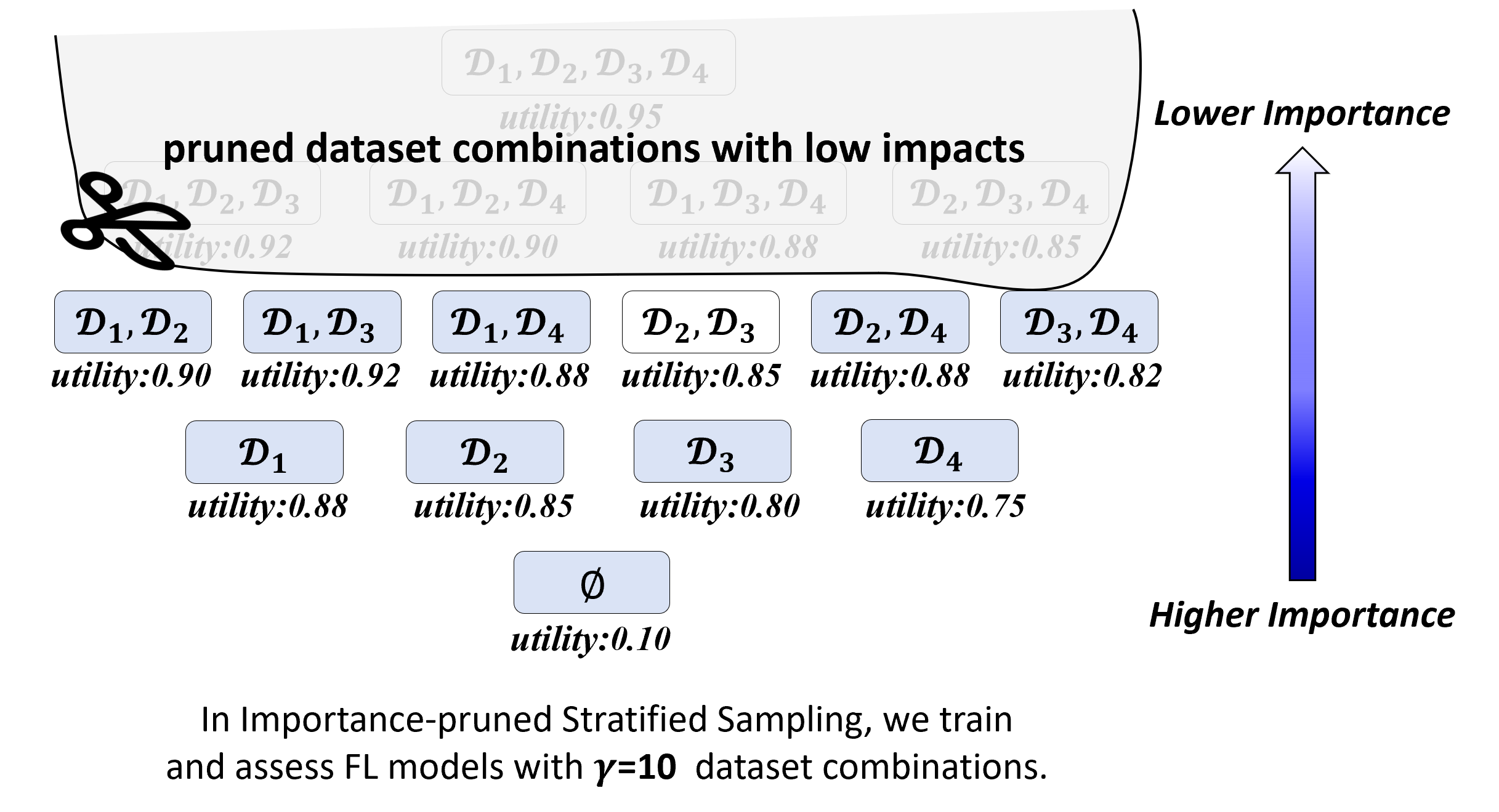}
        \caption{\small Example of \algref{alg:lightSampling} with the same setup as \figref{fig:example_comb_shapley}}
        \label{fig:example_ls}
        \vspace{-1em}
    \end{figure}
    
    \begin{example}
    As in \figref{fig:example_ls}, we further illustrate the \algref{alg:lightSampling} back to the settings in \expref{ex:example_ss_framework} with four FL clients and sampling rounds $\gamma=10$.
    We also take client ${1}$ as the representative.
    Initially, we computed the maximum size for combinations which are fully evaluated in \textsl{IPSS} algorithm, $\medop{k^*=\max\{k\in \mathbb{N}| \sum_{j=0}^{k} \tbinom{4}{j}\leq 10\} = 1}$.
    Then, we train and evaluate FL models with combinations involving no more than one client, \ie $\medop{M_{\emptyset}, M_{\{\mathcal{D}_1\}}, M_{\{\mathcal{D}_2\}}, M_{\{\mathcal{D}_3\}}, M_{\{\mathcal{D}_4\}}}$.
    We can still sample up to $\medop{\gamma-\sum_{j=0}^{k^*}\tbinom{4}{j}=5}$ dataset combinations.
    Satisfying constrains \textit{(1)}-\textit{(3)} in \algref{alg:lightSampling}, we further sample and evaluate FL models under $\medop{\{\mathcal{D}_1, \mathcal{D}_2\}}$, $\medop{\{\mathcal{D}_1, \mathcal{D}_3\}}$, $\medop{\{\mathcal{D}_1, \mathcal{D}_4\}}$, $\medop{\{\mathcal{D}_2, \mathcal{D}_4\}}$ and $\medop{\{\mathcal{D}_3, \mathcal{D}_4\}}$. 
    We take the \textsl{\small MC-SV} to compute average marginal utilities of FL client ${1}$ using all assessed combinations.
    Finally, we have $\medop{\hat{\phi}_{1}=0.22}$.
    Similarly, the estimated data value of FL client $\medop{2}$, $\medop{3}$ and $\medop{4}$ are $\medop{0.20}$, $\medop{0.1842}$ and $\medop{0.1667}$, respectively.
    \end{example}

\subsection{Theoretical Analysis of the IPSS Algorithm}

    \fakeparagraph{Theoretical Evidence}  
    For simplicity, we continue to use the FL linear regression model and theoretically analyze the approximation error and time complexity of the \textsl{IPSS} algorithm.
    \begin{lemma} \label{thm:lmm1}
        Given $n$ datasets each with $t$ training samples,
        if we take negative mean square error (MSE) as the utility function, the estimated data value of client $i$ is $\medop{\mathbb{E}[\hat{\phi}_i] =} \frac{1}{n}(m_{0} - \frac{\mu_{e}|x|}{nt-|x|-1})$, where $|x|$ is input feature dimensions, $\mu_{e}$ is expectation of random noise and $m_0$ is MSE of the initialized model.
    \end{lemma}
        
    \begin{proof} 
        In the proof we follow the analysis framework for FL by \textit{Donahue} and \textit{Kleinberg} \cite{AAAI21_MSG} where all $|\mathcal{D}|$ data items are drawn from the standard \textit{Gaussian distribution} $\mathcal{N}(0,I)$ and the expected MSE of the linear regression model is as,
        \begin{equation}
            \medop{
                \mathbb{E}[mse(|\mathcal{D}|)] = {\mu_{e}|x|}/{(|\mathcal{D}|-|x|-1)},
            }
        \end{equation}
        where $\mu_{e}$ is the expectation of random noise over data, $|x|$ is the dimension of input features, and $|\mathcal{D}|$ is the size of used data.
        Similarly, for a FL linear regression model with $|\mathcal{D}_{S}|=t|S|$ data items, its expected MSE can be writen as,
        \begin{equation}
            \medop{
                \mathbb{E}[U(M_{S})]=\mathbb{E}[mse(|\mathcal{D}_{S}|)] = {\mu_{e}|x|}/{(t|S|-|x|-1)},
            }
        \end{equation}
        
        If we take the negative MSE as the utility function $\medop{U(\cdot)}$ in data valuation for FL, we can calculate the expectation of data value using \textsl{\small MC-SV} and above analysis model \cite{AAAI21_MSG} as below, 
        \begin{equation}
        \begin{aligned}
             &
            \medop{
            \mathbb{E}[\hat{\phi}_i]=\mathbb{E}[\frac{1}{n}\sum_{S\subseteq (N\backslash\{i\})}\frac{U(M_{S\cup\{i\}})-U(M_{S})}{\tbinom{n-1}{|S|}}]
            }
            \\
            &
            \medop{
            =\mathbb{E}[\frac{1}{n}\sum_{S\subseteq (N\backslash\{i\})}\frac{-mse((|S|+1)t)+mse(|S|t)}{\tbinom{n-1}{|S|}}]
            }\\
            &
            \medop{
            =\frac{1}{n}\sum_{k=0}^{n-1}(-\mathbb{E}[mse((k+1)t)]+\mathbb{E}[mse(kt)]) 
            }\\ 
        \end{aligned}
        \end{equation}
        As $mse(0)$ is not defined in \cite{AAAI21_MSG}, we let $m_0$ denotes the MSE of the initialized linear model.
        So the $\mathbb{E}[\hat{\phi}_i]$ is ,
        \begin{equation}
        \begin{aligned}
            \medop{
            \mathbb{E}[\hat{\phi}_i]
            }
            &
            \medop{
                =\frac{1}{n}(m_0 - \mathbb{E}[mse(nt)])=\frac{1}{n}(m_0 - \frac{\mu_{e}\cdot|x|}{nt-|x|-1})
            }
        \end{aligned}         
        \end{equation}
     Finally, we have the $\mathbb{E}[\hat{\phi}_i]$ and then completed our proof.
    \end{proof}

    \begin{theorem} \label{thm:thm4}
    Given the sampling rounds $\medop{\gamma}$ and taking same assumption as \lemref{thm:lmm1}, the approximation error bound of \algref{alg:lightSampling} is $\mathcal{O} (\frac{n - k^{*}}{k^{*}nt})$, where $k^{*} = {\underset{k}{{\arg\max}}\{{\sum_{i=0}^{k} \binom{n}{j} \leq \gamma \}}}$.
    

    \end{theorem}
    \begin{proof}
        \algref{alg:lightSampling} takes all dataset combinations with no more than $\medop{k^{*}}$ clients, where $\medop{k^{*} =\underset{k}{{\arg\max}}\{{\sum_{i=0}^{k} \binom{n}{j} \leq \gamma\}}}$.
        Similar to \lemref{thm:lmm1}, we calculate FL client $i$'s expected contribution as,
        \begin{equation}
        \begin{aligned}
            \medop
            {
            \mathbb{E}[\hat{\phi}_{i}^{k^{*}}]
            }
            &
            \medop{=\frac{1}{n}\sum_{k=0}^{k^{*}-1}(-\mathbb{E}[mse((k+1)t)]+\mathbb{E}[mse(kt)])
            }
            \\
            &
            \medop
            {
            =\frac{1}{n}(m_{0} - \frac{\mu_{e}|x|}{k^{*}t-|x|-1})
            }
        \end{aligned}
        \end{equation}
    Together with \lemref{thm:lmm1}, the ratio of $\mathbb{E}[\hat{\phi}_{i}^{k^{*}}]$ and $\mathbb{E}[\phi_{i}]$ is,
        \begin{equation} \label{thm:thm4_step_22}
        \begin{aligned}
        \medop
        {
           \frac{\mathbb{E}[\hat{\phi}_i^{k^{*}}]}{\mathbb{E}[\phi_i]} = \frac{\frac{1}{n}(m_{0} - \frac{\mu_{e}\cdot|x|}{k^{*}t-|x|-1})}{\frac{1}{n}(m_{0} - \frac{\mu_{e}\cdot|x|}{nt-|x|-1})}
           = 1 - \frac{\frac{\mu_{e}\cdot|x|}{k^{*}t-|x|-1}-\frac{\mu_{e}\cdot|x|}{nt-|x|-1}}{m_{0}-\frac{\mu_{e}\cdot|x|}{nt-|x|-1}}
        }
        \end{aligned}
        \end{equation}
        As a model trained by $\medop{|x|+2}$ training samples can outperform the initialized model, so the $\medop{mse(|x|+2)}$ is less than the MSE of the initialized model $m_0$.
        We have following inequations,
        \begin{equation}
        \begin{aligned}
        \medop
        {
            \frac{\mathbb{E}[\hat{\phi}_i^{k^{*}}]}{\mathbb{E}[\phi_i]}
        }
        &
        \medop{
            \geq 1 - \frac{\frac{1}{k^{*}t-|x|-1}-\frac{1}{nt-|x|-1}}{\frac{1}{|x|+2-|x|-1}-\frac{1}{nt-|x|-1}}
        } 
        {\small 
            \medop{
            = 1-\mfrac{(n-k^{*})t}{(k^{*}t-|x|-1)(nt-|x|-2)}
            } 
        }
        \end{aligned}
        \end{equation}
        Note that the input feature dimension $|x|$ is an constant number and we can complete the proof of \thmref{thm:thm4} as follows,
        \begin{equation}
        \medop
        {
            \frac{|\mathbb{E}[\hat{\phi}_i^{k^{*}}]-\mathbb{E}[\phi_{i}]|}{\mathbb{E}[\phi_i]} 
        }
        \medop{
            \leq \mfrac{(n-k^{*})t}{(k^{*}t-|x|-1)(nt-|x|-2)} = \mathcal{O}\left(\frac{n-k^{*}}{k^{*}nt}\right) 
        }
        \end{equation}
    \end{proof}

    \fakeparagraph{Approximation Error Analysis}
    \thmref{thm:thm4} suggests that even with a small $k^{*}$, the relative error between $\mathbb{E}[\phi_{i}^{k^{*}}]$ and $\mathbb{E}[\phi_{i}]$ can remain minimal, which is particularly relevant in typical FL scenarios, where the number of training samples in a dataset substantially exceeds the dimension of input features. 
    Take MNIST \cite{mnist}, a most widely used benchmark dataset, as an example.
    It contains over 60,000 training images, each represented by 784 dimensional features, implying that $|x| \ll nt$. 
    If each client holds the same number of training samples, the relative error of \algref{alg:lightSampling} is $\mathcal{O}(\frac{n}{k^{*}\mathcal{D}})$. 
    In this case, accurately approximating the data value only need to evaluate a select group of dataset combinations, consistent with the \textit{key combinations phenomenon} observed earlier in \secref{sec:improve_KCP}.

    \fakeparagraph{Time Complexity Analysis}
    The time cost of \algref{alg:lightSampling} mainly relies on FL training and assessing processes on various dataset combinations.
    Assuming time cost to train and evaluate the FL model is denoted as $\tau$. The time complexity of \algref{alg:lightSampling} can be analyzed as follows.
    For lines 1-14, the complexity is $\mathcal{O}(\tau\gamma)$, as \algref{alg:lightSampling} utilizes no more than $\gamma$ dataset combinations.
    For lines 15-17, \algref{alg:lightSampling} computes marginal contributions upto $\sum_{j=0}^{k^{*}-1}\tbinom{n}{j+1}\tbinom{n}{j}+\tbinom{n}{k^{*}}(\gamma-\sum_{j=0}^{k^*}\tbinom{n}{j})$ times, leading to a time complexity of $\mathcal{O}(\gamma\tbinom{n}{k^*})$.
    We usually only take a small group of combinations in data valuation so $k^*$ is a small integer.
    As time to train and evaluate a FL model $\tau$ is usually much larger than $\tbinom{n}{k^*}$, the time complexity of the \textsl{IPSS} is $\mathcal{O}(\tau\gamma)$.
    
\section{Experimental Evaluations} \label{sec:exp}
This section presents evaluations of our proposed methods.
\subsection{Experimental Setup}
\fakeparagraph{Datasets}
    We evaluate baseline algorithms for SV-based data valuation in FL over both synthetic and real-world datasets.
    
    \noindent\textit{(i) \underline{Synthetic Dataset}}.  
    We take the MNIST~\cite{mnist}, a widely-used datasets with 60,000+ training samples and 10,000+ testing samples to create the synthetic datasets in FL.        
    Following the experimental setup in \cite{Bigdata19_SV, TIST22_GTG}, we split the MNIST~\cite{mnist} into partitions and create customized training dataset tailored for FL, where datasets of each FL client varies in size, distribution and quality (\ie noise). 
    We highlight experimental features in each FL setting. 
    \textit{(a) same-size-same-distribution}: we split training dataset into partitions with same size and label distribution.
    \textit{(b) same-size-different-distribution}: we partition the training samples and set some label are majorly belongs to certain client.
    \textit{(c) different-size-same-distribution}: we randomly split
    training samples into partitions with their ratios of data size {\small${1:2:\cdots:n}$}, where $n$ is the client number.  
    \textit{(d) same-size-noisy-label}: we change {\small${0\%\sim20\%}$} of labels in the partitioned dataset into one of other labels with equal probability.
    \textit{(e) same-size-noisy-feature}: we generate \textit{Gaussian} noise $\mathcal{N}(0,1)$ and scale them by multiplying {\small ${0.00 \sim 0.20}$} as the noise added on training samples.
    
    \noindent\textit{(ii) \underline{Real-world Dataset}}. We also conduct experiments on three real-world datasets, FEMNIST \cite{arXiv18_LEAF}, Adult \cite{adult_2} and Sent-140 \cite{arXiv18_LEAF}.
    FEMNIST \cite{arXiv18_LEAF} is a benchmark dataset for FL and is included within TensorFlow Federated \cite{tensorflow_f}.
    It contains 805,000+ samples from 3,500+ users, allowing it to be partitioned into datasets for clients in FL by the user-ids.
    The Adult~\cite{adult_2} is a tabular dataset commonly used in vertical FL \cite{SIGMOD21_FL_Tree, fu2022blindfl, ICDE22_SV_Wang} and it contains 48,800+ training samples and 14 features (\eg income, occupation, and native-country).
    Without loss of generality, we can partition the training samples in Adult to several datasets for FL clients according to user's occupation.
    
\fakeparagraph{Compared Algorithms} 
We compare our \textsl{IPSS} algorithm (in \secref{sec:ps_alg}) with a series of existing baselines (in three categories). 
The first category (\textsl{Perm-Shapley}, \textsl{MC-Shapley} and \textsl{DIG-FL}) calculates data value directly by definition, while the second category (\textsl{Exteneded-TMC}, \textsl{Extended GTB} and \textsl{CC-Shapley}) uses sampling-based methods.
The last category (\textsl{OR}, $\lambda$\textsl{-MR}, \textsl{GTG-Shapley}) approximates the data value through gradients collected in FL training process.

\begin{itemize}

    \item \textsl{Perm-Shapley.}
    It directly calculates data value of  clients in FL according to the definition of the \textit{Permutation-based Shapley value} (\textsl{\small Perm-SV}), which trains and evaluates FL models based on permutations of all datasets.
    
    \item \textsl{MC-Shapley}. 
    Similarly, it directly calculates the data value through the \textsl{MC-SV} based computation scheme.
    \item  \textsl{DIG-FL}.
    It efficiently approximates the data value in FL~\cite{ICDE22_SV_Wang}, which only needs to evaluate $\mathcal{O}(n)$ numbers of dataset combinations under certain assumptions \cite{ICDE22_SV_Wang}.
    \item \textsl{Extended-TMC}. 
    It is an extension of widely-adopted data valuation scheme for general machine learning \cite{icml19shapley}.
    We extend and compare the Truncated Monte Carlo algorithm of \cite{icml19shapley} to FL scenario. 
    It randomly generates a permutation $\pi$ of all $n!$ permutations and trains and evaluates the FL models based on the permutation.
    Then, the algorithm approximates the \textsl{Perm-SV} according to, 
    \begin{equation}
        {
        \setlength{\abovedisplayskip}{5pt}
        \setlength{\belowdisplayskip}{5pt}
            \medop{
    	   \phi_i=\mathbb{E}_{\pi\sim\Pi}[U(M_{\pi[\mathrm{p}(i)]\cup\{i\}})-U(M_{\pi[\mathrm{p}(i)]})].
            }
        }
    \end{equation}
    
    \item \textsl{Extended-GTB.}
    It is also an extension of a representative data valuation scheme \cite{aistats19shapley} and we extend the Group-Testing-Based SV estimation to FL scenario as follows.
    The \textsl{GTB} can estimate the contributions of each client in FL by solving a feasibility problem \rev{through the randomly selected subsets $S\subseteq N$}.
    Finally, we incrementally relax the constraints until there is a feasible solution.
 
     \item \textsl{OR}. It directly takes gradients within the FL process with all clients the same as gradients under other combinations \cite{Bigdata19_SV}.
     OR can approximate the FL model by these gradients without extra training, however, there is no theoretical guarantee for OR in approximation errors.

    \item \textsl{$\lambda$-MR}. It takes the \textsl{MC-SV}-based scheme and estimates data value in each training round of FL and aggregate them as the final results \cite{FLIP20_Wei}. 
    The $\lambda$-MR avoids the additional training of the FL models as well. 
    
    \item  \textsl{CC-Shapley.}
    It is one of the state-of-the-art sampling methods to approximate the Shapley value \cite{SIGMOD23_SV_Zhang}, which estimates data value using the \textsl{\small CC-SV}-based schemes.

    \item \textsl{GTG-Shapley.}
    Similar to $\lambda$-MR, it also approximates the data value using gradients \cite{TIST22_GTG}.
    It adopts the \textsl{Perm-SV} and uses Monte Carlo sampling approach to reduce the number of model reconstructions over rounds.
\end{itemize}

\begin{figure*}[htb]
    \centering
    \setlength{\tabcolsep}{0.5pt}
    \begin{tabular}{ccccc}
        \multicolumn{5}{c}{
            \vspace{-0.5em}
            \includegraphics[width=0.985\linewidth]{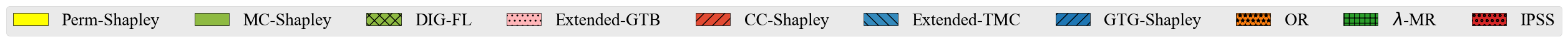}
        } \\
        \vspace{-0.5em}
        \includegraphics[width=0.185\linewidth] {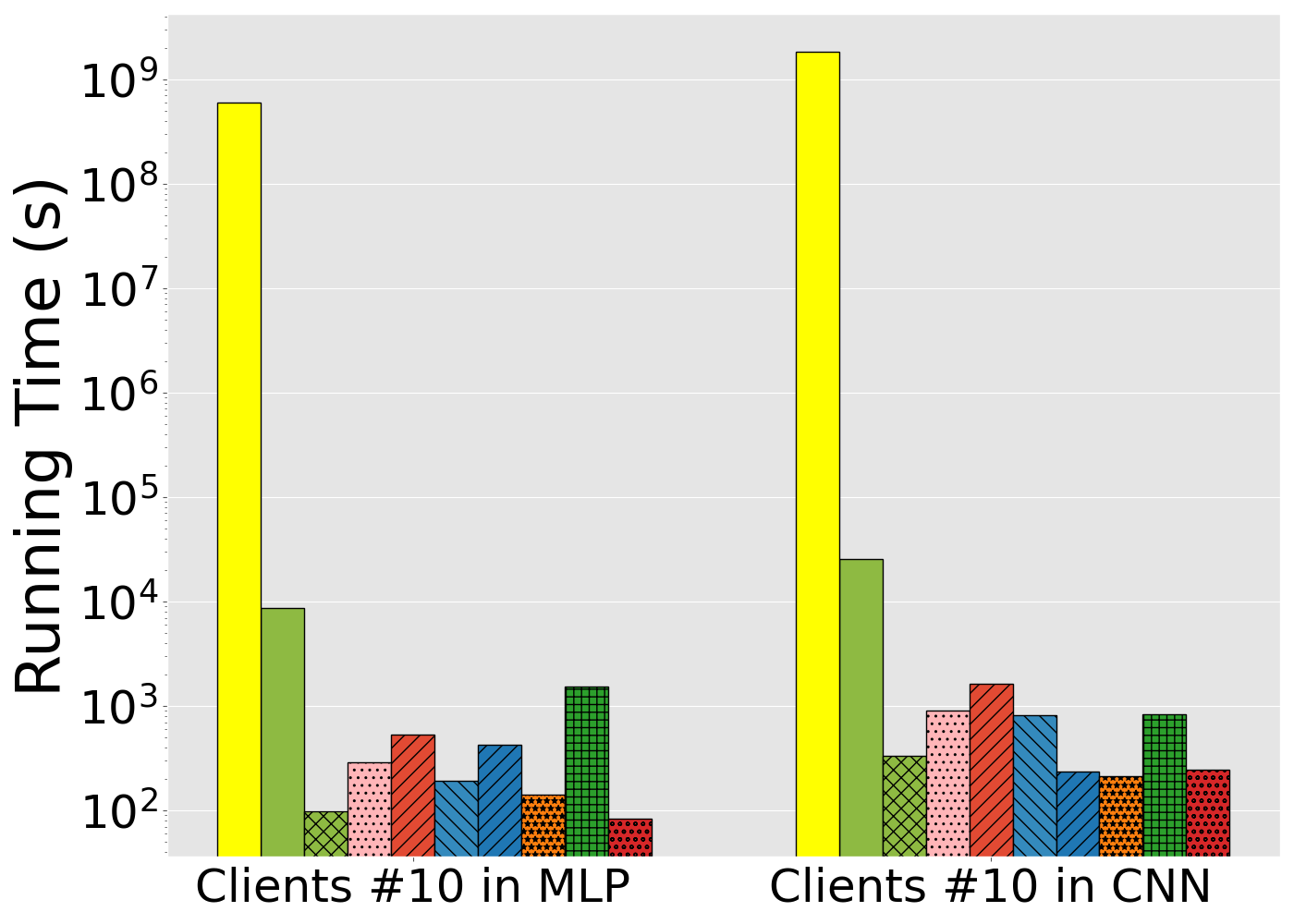} & 
        \includegraphics[width=0.185\linewidth]{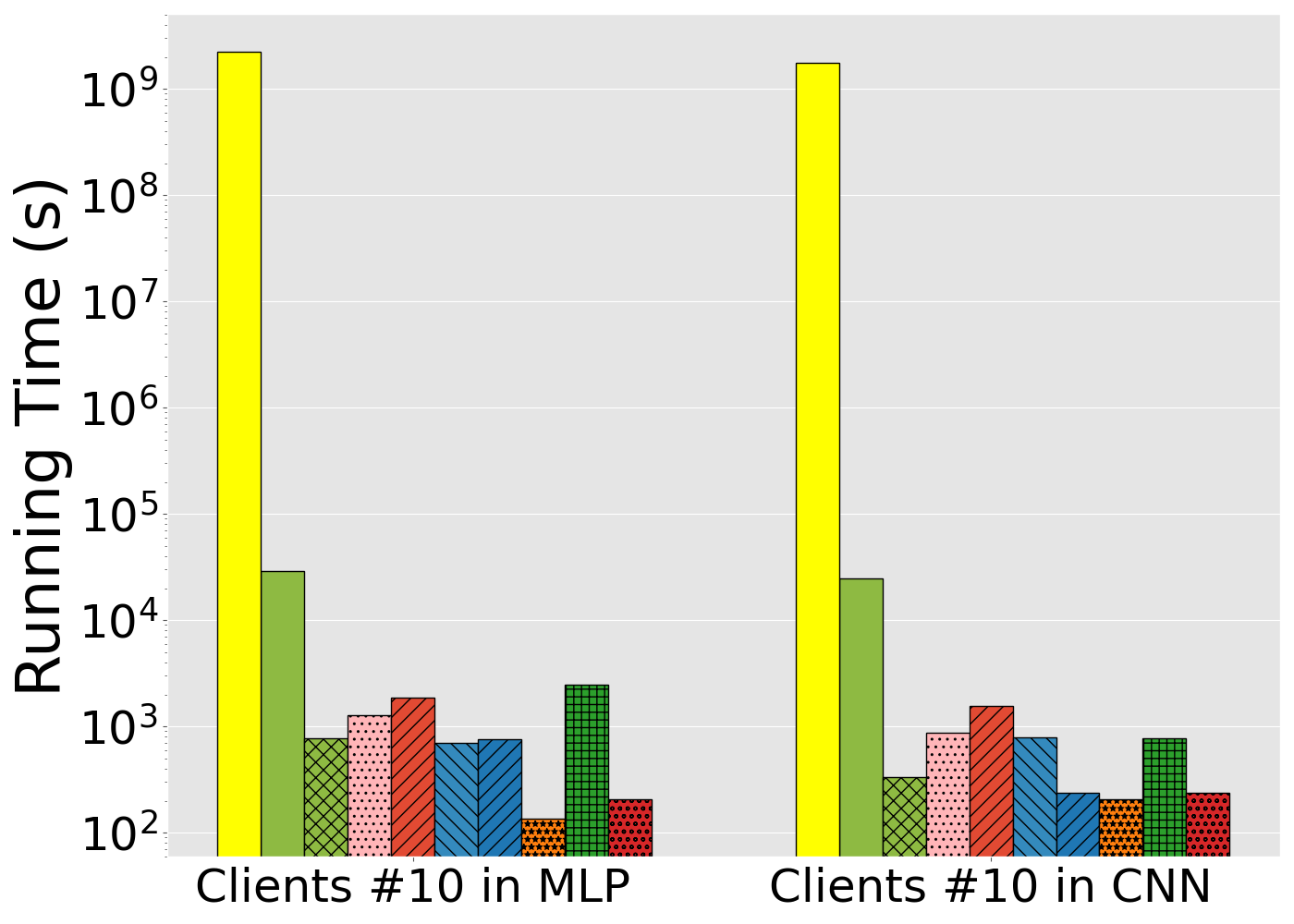} & 
        \includegraphics[width=0.185\linewidth]{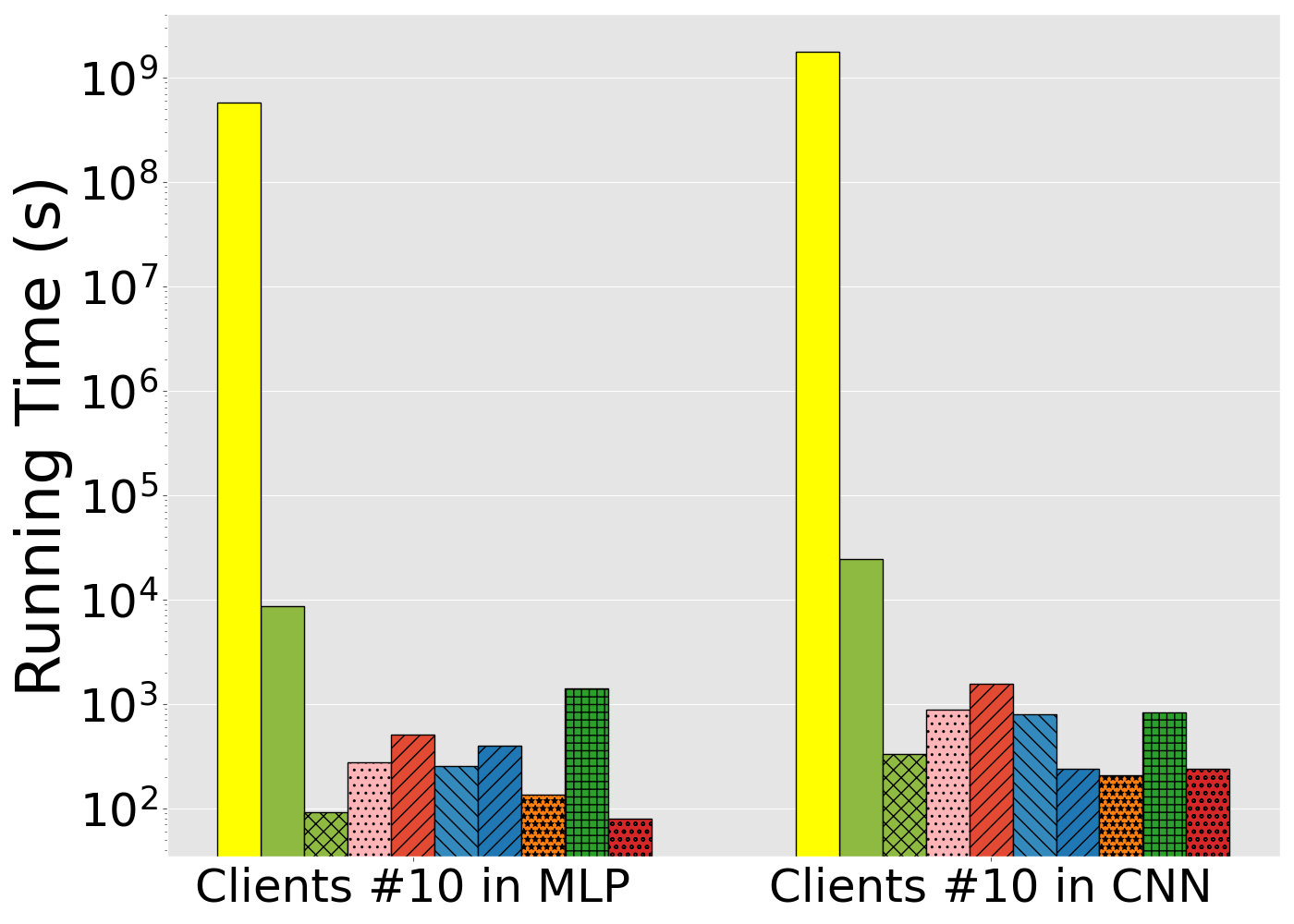} & 
        \includegraphics[width=0.185\linewidth]{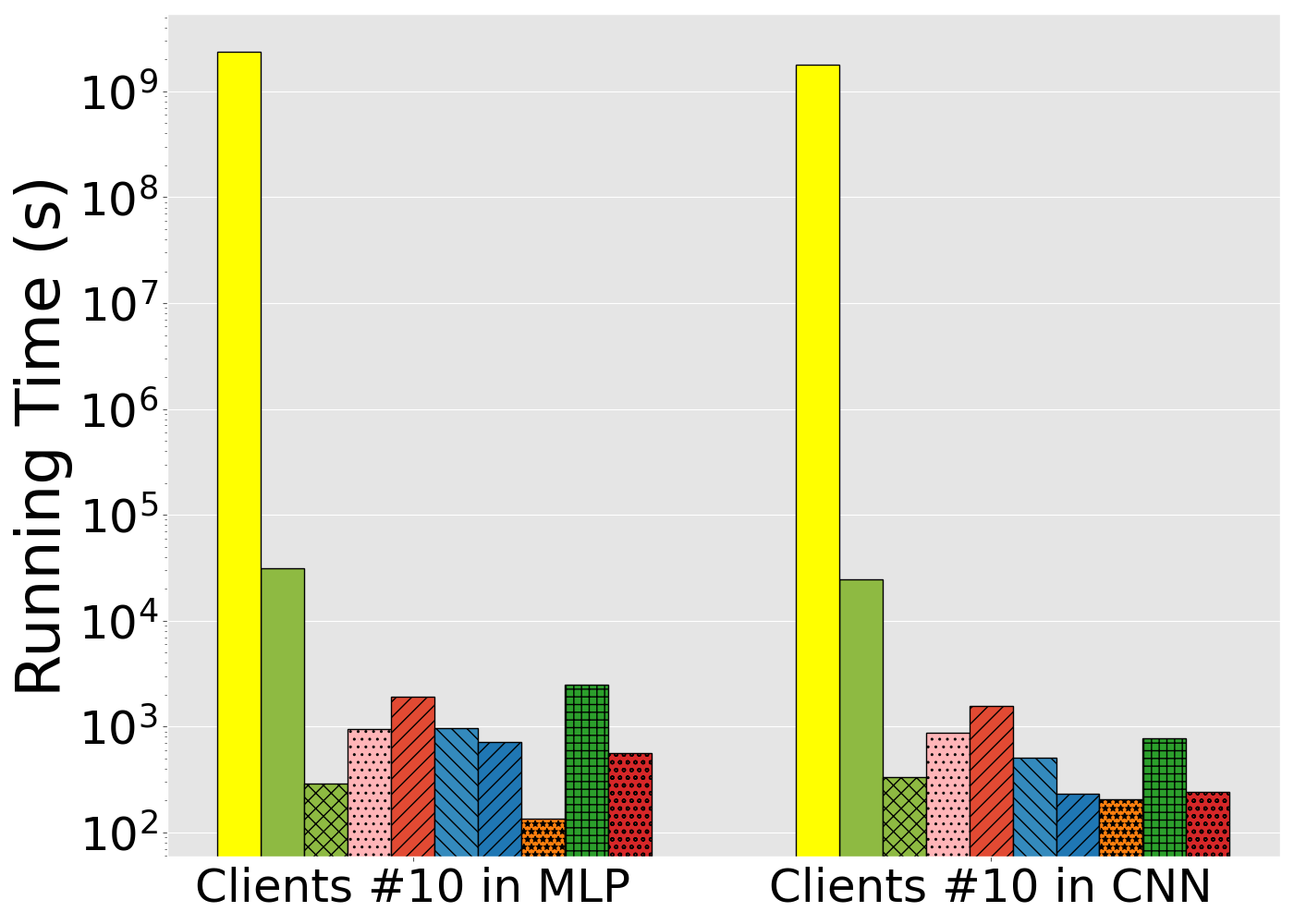} & 
        \includegraphics[width=0.185\linewidth]{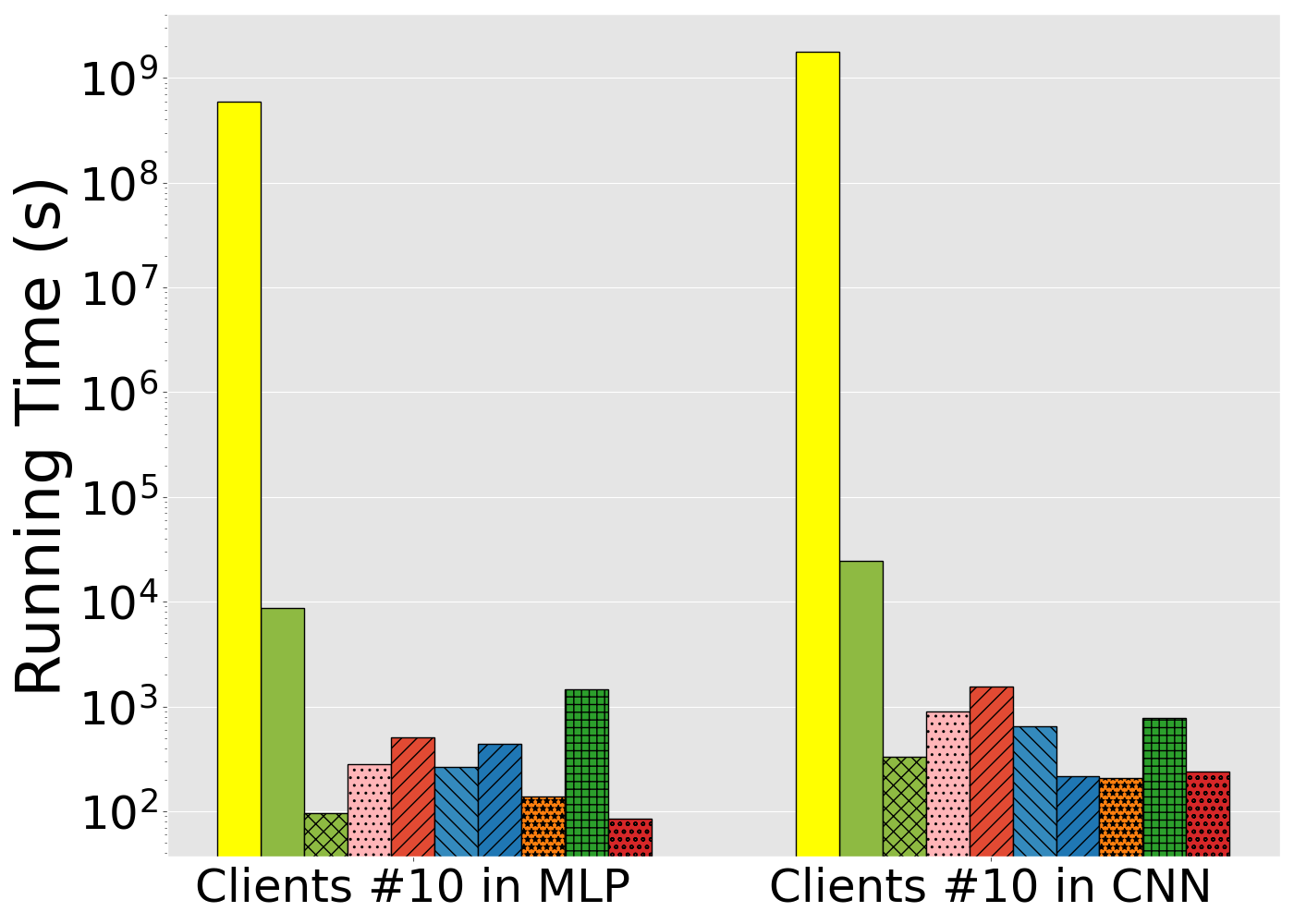} \\ 
        \vspace{-0.5em}
        \includegraphics[width=0.185\linewidth]{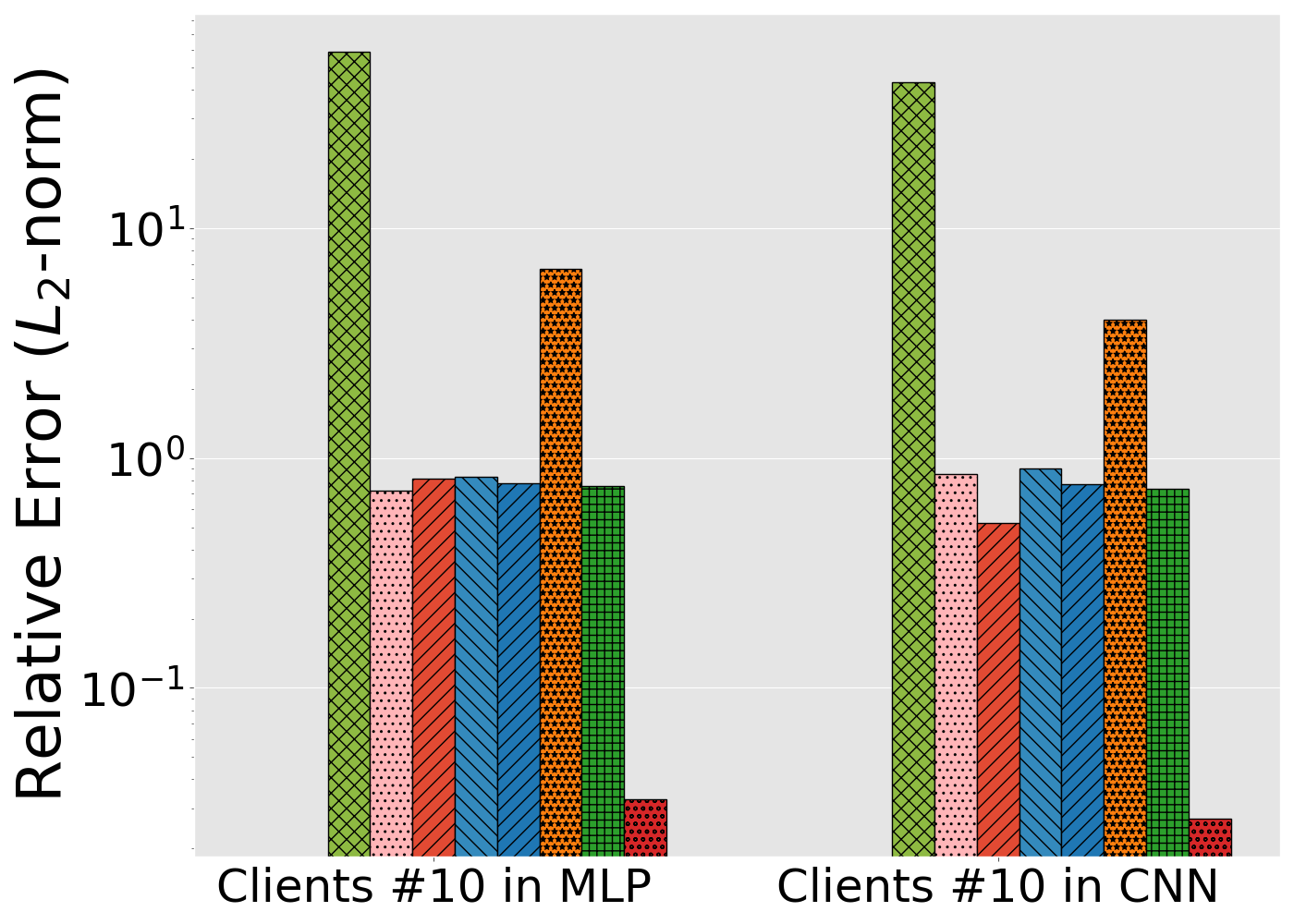} & 
        \includegraphics[width=0.185\linewidth]{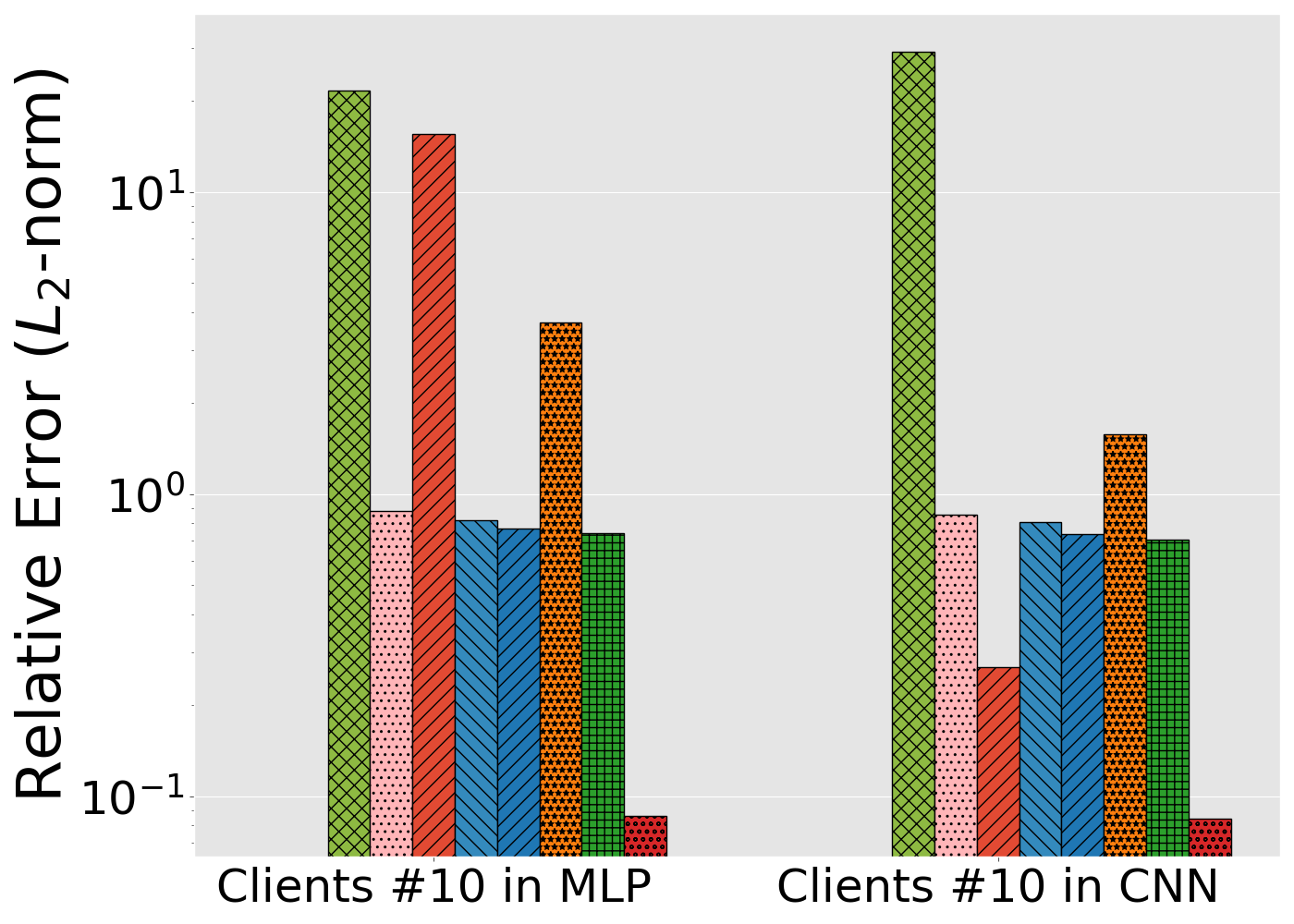} & 
        \includegraphics[width=0.185\linewidth]{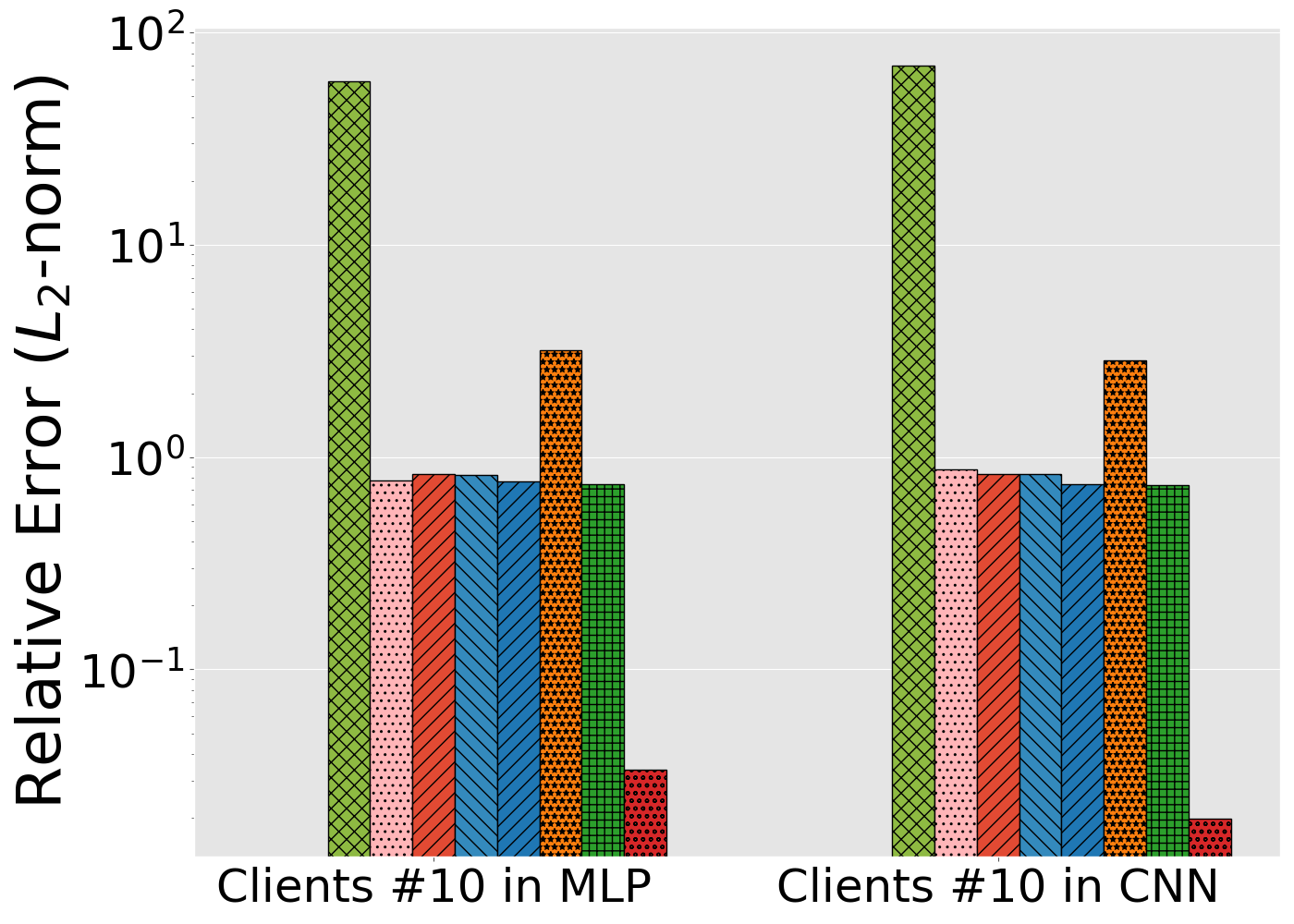} & 
        \includegraphics[width=0.185\linewidth]{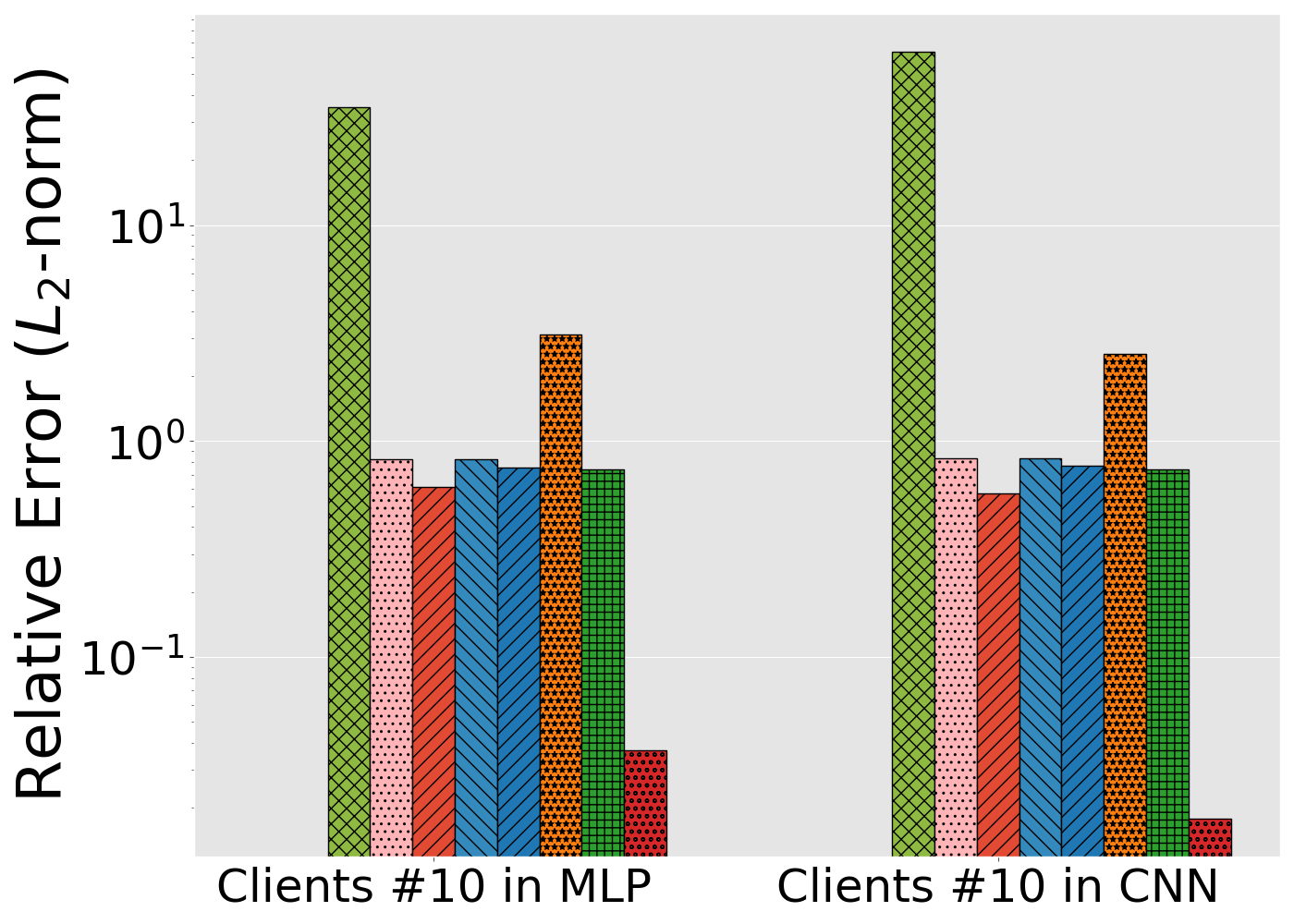} & 
        \includegraphics[width=0.185\linewidth]{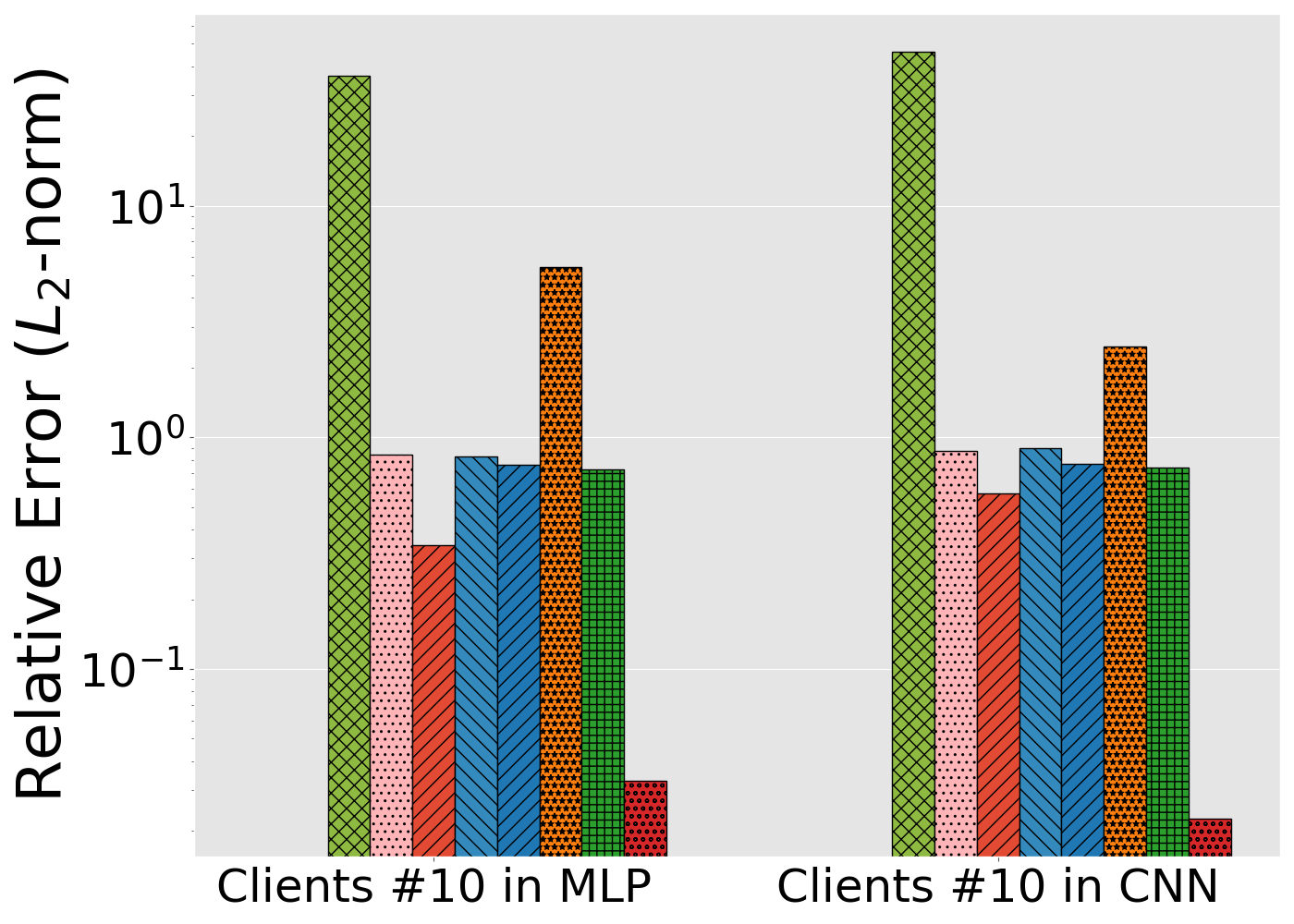} \\ 
         \textit{{\small (a) same-size-same-distr.}} & \textit{{\small (b) same-size-diff.-distr.}} & \textit{{\small (c) diff.-size-same-distr.}} & \textit{{\small (d) same-size-noisy-label}} & \textit{{\small (e) same-size-noisy-feature}} \\
    \end{tabular}
    \caption{Experimental results on the synthetic datasets with five different setups varying in size, distribution and quality.}
    \label{fig:e2e}
\end{figure*}

\fakeparagraph{Evaluation Metrics} 
We employ following two metrics to assess the performance of compared algorithms.
\textit{(i) Calculation Time}: it measures the running time required to calculate the data value, including the time to train and evaluate FL models. 
\textit{(ii) Approximation Error}: it represents the effectiveness of approximation algorithms by the relative error in $\textit{l}_{2}$-norm:
\rev{
\begin{equation}
        \medop{
         \textit{l}_{2}(\hat{\phi}, \phi) = 
         {\Vert\hat{\phi}-\phi\Vert_{2}}/{\Vert\phi \Vert_{2}} = 
         {\sqrt{\sum_{i=1}^{n} (\phi_{i}-{\hat{\phi}_{i}})^{2}}}/{\sqrt{\sum_{i=1}^{n} {\phi_{i}}^{2}}}
        }
\end{equation}
}
    where $\phi=(\phi_1, \phi_2, \dots, \phi_n)$ denotes the data value of $n$ FL clients and $\hat{\phi}=(\hat{\phi}_1, \hat{\phi}_2, \dots, \hat{\phi}_n)$ is the approximation results.

\fakeparagraph{Implementations}
All algorithms were implemented in Python with TensorFlow 2.4 \cite{tensorflow} and TensorFlow Federated 0.18 \cite{tensorflow_f}. 
To simulate multiple data providers in FL, we adopt the multi-processing techniques and the gRPC protocol. 
The experimental setup was executed on a machine equipped with an NVIDIA GeForce RTX 3090 GPU, an AMD Ryzen 7950X CPU @ 3.0GHz, and 128GB of main memory. 
Our experiments incorporated multi-layer perceptron (MLP), convolutional neural network (CNN) and XGBoost (XGB) models, which are all extensively used in data science community.
\rev{
    When the number of FL clients is three, six, and ten, all sampling-based approximation approaches are configured with the same number of sampling rounds, \ie 5, 8, and 32, respectively (as in \tabref{tlb:gamma}).
}
\rev{
    The open-sourced code is available at ``{https://github.com/t0ush1/Shapley-Data-Valuation}".
}
\begin{table}[thb]
    \centering
    \belowrulesep=0.2pt
    \aboverulesep=0.2pt
    \resizebox{0.95\linewidth}{!}{
        \begin{tabular}{|c|c|c|}
        \hline
        
        \rev{$n=3 \rightarrow \gamma=5$} & \rev{$n=6 \rightarrow\gamma=8$} & \rev{$n=10 \rightarrow \gamma=32$}
        \\
        \hline 
        \end{tabular}
    }
    \caption{\small \rev{The adopted sampling rounds ($\gamma$) for client number ($n$).}}
    \label{tlb:gamma}
    \vspace{-1.8em}
\end{table}

\subsection{Performance on Synthetic Datasets}

We showcase the experimental results under varying dataset sizes, distributions, and noise levels. 
This series of experiments presents the time cost and approximation error for the compared algorithms using both MLP and CNN models.
In each of the following five training setups, we use ten clients in FL. 


   

\noindent\rev{\textit{(a) same-size-same-distr.}}.
\figref{fig:e2e}(a) plots time cost and error of compared algorithms.
For the time cost, \textsl{OR} and \textsl{IPSS} have the lowest time cost in both MLP model and CNN model.
The time cost of \textsl{MC-Shapley} is 62.2$\times$ and 104.1$\times$ as \textsl{OR} and \textsl{IPSS}, respectively.
For approxi error, \textsl{IPSS} is the lowest and close to zero, \ie the estimated data value via \textsl{IPSS} is almost the same as the exact one.
\noindent\rev{\textit{(b) same-size-diff.-distr.}}.
From \rev{\figref{fig:e2e}(b)}, \textsl{OR} is still the fastest in both MLP and CNN model and \textsl{IPSS} is the second fastest.
\textsl{IPSS} is 3.4$\sim$9.0$\times$ faster than \textsl{GTG-Shapley} and \textsl{CC-Shapley}, respectively.
For approximation error, \textsl{IPSS} outperforms other algorithms.
\textsl{OR} performs poor in accuracy.
Overall, \textsl{IPSS} outperforms the others in this setting.
\noindent\rev{\textit{(c) diff.-size-same-distr.}}.
As shown in \rev{\figref{fig:e2e}(c)}, \textsl{OR} and \textsl{IPSS} exhibit a lower time cost compared to other baseline algorithms.
For estimation error, \textsl{IPSS} also approximates the exact SV well and outperforms the other approximation algorithms significantly. 
The $\lambda$-\textsl{MR} ranks the second in accuracy for both MLP and CNN model.
\noindent\rev{\textit{(d) same-size-noisy-label}}.
From \rev{\figref{fig:e2e}(d)},
the relative error of $\lambda$-\textsl{MR} and \textsl{IPSS} is stable and \textsl{IPSS} still has the lowest error.
The relative error of \textsl{Extended-TMC} and \textsl{Extended-GTB} is 22.3$\times$ and 22.5$\times$ of \textsl{IPSS}, respectively.
\noindent\rev{\textit{(e) same-size-noisy-feature}}.
$\lambda$-\textsl{MR} and \textsl{CC-Shapley} have the highest time cost for MLP and CNN models among all compared algorithms.
As shown in \rev{\figref{fig:e2e}(e)}, the error of \textsl{CC-Shapley} and $\lambda$-\textsl{MR} is 10.3$\sim$26.0$\times$ and 22.1$\sim$33.6$\times$ greater than that of \textsl{IPSS}, respectively.

\subsection{Results on Real-world Dataset} \label{sec:exp_real}
    We also validate our approximation algorithm on two real dataset, which can be naturally partitioned several datasets for clients in FL and we detail evaluations on each below.

\begin{table}[bht]
\belowrulesep=0.10pt
\aboverulesep=0.10pt
\setlength{\tabcolsep}{0.75pt}
\resizebox{\linewidth}{!}{
    \begin{tabular}{@{}c|c|c|cccccccccc@{}}
    \toprule [1.2pt]
     & {$n$} & {Metrics} & {\textsl{Perm-Shap.}} & {\textsl{MC-Shap.}} & {\textsl{DIG-FL}} & {\textsl{Ext-TMC}} & {\textsl{Ext-GTB}} & {\textsl{CC-Shap.}} & {\textsl{GTG-Shap.}} & {\textsl{OR}} & {\textsl{$\lambda$-MR}} & {\textsl{IPSS}} \\ \midrule[1pt]
    \multirow{6}{*}{\textbf{\begin{tabular}[c]{@{}c@{}}MLP\\ \end{tabular}}} & \multirow{2}{*}{3} & Time{\small (s)} & 3729 & 842 & 584 & 568 & 807 & 1021 & 47 & \textbf{\colorbox{dc1}{12}}  & 29 & 258 \\
     &  & Error{\small ($l_2$)} & - & - & 5.01 & 0.79 & 0.59 & 0.35 & 0.90 & 2.46 & 0.88 & \textbf{\colorbox{dc1}{0.06}} \\ \cmidrule(){2-13} 
     & \multirow{2}{*}{6} & Time{\small (s)} & 9.1$\times 10^6$ & 6496 & 1077 & 843 & 1120 & 2020 & 161 & \textbf{\colorbox{dc1}{89}} & 228 & 329 \\
     &  & Error{\small ($l_2$)} & - & - & 0.70 & 0.96 & 0.90 & 1.93 & 0.89 & 3.13 & 0.87 & \textbf{\colorbox{dc1}{0.49}} \\ \cmidrule(){2-13} 
     & \multirow{2}{*}{10} & Time{\small (s)} & 6.8$\times 10^9$ & 95985 & 1695 & 3061 & 4129 & 5988 & 1086 & 1414 & 3764 & \textbf{\colorbox{dc1}{568}} \\
     &  & Error{\small ($l_2$)} & - & - & 0.77 & 0.82 & 0.85 & 1.16 & 0.85 & 3.09 & 0.83 & \textbf{\colorbox{dc1}{0.02}}\\ \midrule 
    \multirow{6}{*}{\textbf{\begin{tabular}[c]{@{}c@{}}CNN\\ \end{tabular}}} & \multirow{2}{*}{3} & Time{\small (s)} & 1629 & 372 & 230 & 231 & 352 & 413 & 26 & \textbf{\colorbox{dc1}{7}} & 22 & 142 \\
     &  & Error{\small ($l_2$)} & - & - & 95.14 & 0.81 & 0.60 & 0.02 & 0.87 & 0.46 & 0.73 & \textbf{\colorbox{dc1}{0.01}} \\ \cmidrule(){2-13} 
     & \multirow{2}{*}{6} & Time{\small (s)} & 3.6$\times 10^5$ & 2783 & 407 & 352 & 484 & 667 & 108 & \textbf{\colorbox{dc1}{47}}  & 154 & 211 \\
     &  & Error{\small ($l_2$)} & - & - & 78.25 & 0.91 & 0.70 & 0.40 & 0.76 & 0.35 & 0.73 & \textbf{\colorbox{dc1}{0.02}} \\ \cmidrule(){2-13} 
     & \multirow{2}{*}{10} & Time{\small (s)} & 2.8$\times 10^9$ & 40134 & 655 & 1220 & 1612 & 2553 & 680 & 641 & 2504 &\textbf{\colorbox{dc1}{257}}  \\
     &  & Error{\small ($l_2$)} & - & - & 98.42 & 0.83 & 0.87 & 2.60 & 0.75 & 0.76 & 0.71 & \textbf{\colorbox{dc1}{0.02}} \\ \bottomrule 
    \end{tabular}
    }
    \caption{\small 
    We mark the ``best performance'' as \textbf{\colorbox{dc1}{\textbf{green}}}. 
    ``-'' denotes the solution can exactly computes the \textsl{SV}-based data values. 
    }
    \label{tab:real_exp_res_1}
    \vspace{-1.5em}
\end{table}
     
    \fakeparagraph{Results on FEMNIST} 
    \tabref{tab:real_exp_res_1} shows the experimental results on FEMNIST \cite{arXiv18_LEAF} datasets across various numbers of FL clients and we take both MLP and CNN as the FL model.
    
    \noindent\textit{In MLP model}.
        Taking MLP as the FL model, we have the following observations:
        \textit{(i)} In scenarios with ten FL clients, our \textsl{IPSS} algorithm achieves the lowest time cost, reducing computing overhead by 99\% compared to \textsl{MC-Shapley} and performing 2.98$\times$ and 1.91$\times$ faster than \textsl{DIG-FL} and \textsl{GTG-Shapley}, respectively.
        \textit{(ii)} In terms of the relative error, \textsl{IPSS} significantly outperforms other algorithms across all numbers of clients.
        Notably, the error of \textsl{IPSS} is 38.5$\times$ and 42.5$\times$ lower than \textsl{Extended-TMC} and \textsl{GTG-Shapley} with 10 clients.
        
        \noindent\textit{In CNN model}.
        The results in CNN model exhibit similarities to that observed in MLP model.
        \textit{(i)} For the efficiency, \textsl{OR} is superior when the number of clients is 3 and 6, while \textsl{IPSS} is the fastest when there are larger number of clients.  
        \textit{(ii)} Regarding approximation error, \textsl{IPSS} consistently shows the lowest error over various clients, which is one order of magnitude smaller than other approximation algorithms. 
        \textit{(iii)} However, the relative error of \textsl{DIG-FL} is notably higher in the CNN model. 
        \rev{\textit{In summary, \textsl{IPSS} excels in efficiency with more FL clients and consistently exhibits lower error compared to baselines across various numbers of FL clients.}}

    \begin{table}[htb] 
    \belowrulesep=0.10pt
    \aboverulesep=0.10pt
    \setlength{\tabcolsep}{0.7pt}
    \resizebox{\linewidth}{!}{
        \begin{tabular}{@{}c|c|c|cccccccccc@{}}
        \toprule[1.2pt]
         & {$n$} & {Metrics} & {\textsl{Perm-Shap.}} & {\textsl{MC-Shap.}} & {\textsl{DIG-FL}} & {\textsl{Ext-TMC}} & {\textsl{Ext-GTB}} & {\textsl{CC-Shap.}} & {\textsl{GTG-Shap.}} & {\textsl{OR}} & {\textsl{$\lambda$-MR}} & {\textsl{IPSS}} \\ \midrule[1pt]
        \multirow{6}{*}{\textbf{\begin{tabular}[c]{@{}c@{}}MLP\\ \end{tabular}}} & \multirow{2}{*}{3} & Time{\small (s)} & 720 & 164 & 94 & 95 & 138 & 199 & 59 & \textbf{\colorbox{dc1}{13}} & 48 & 69 \\
         &  & Error{\small ($l_2$)} & - & - & 1.02 & 1.46 & 1.89 & 0.09 & 5.30 & 1.00 & 2.93 & \textbf{\colorbox{dc1}{0.05}} \\ \cmidrule(){2-13} 
         & \multirow{2}{*}{6} & Time {\small(s)} & 3.3$\times 10^5$ & 2820 & 252 & 220 & 306 & 530 & 271 & \textbf{\colorbox{dc1}{74}} & 347 & 146 \\
         &  & Error{\small ($l_2$)} & - & - & 1.12 & 2.30 & 2.02 & 0.18 & 3.65 & 1.00 & 3.21 & \textbf{\colorbox{dc1}{0.13}} \\ \cmidrule(){2-13} 
         & \multirow{2}{*}{10} & Time{\small (s)} & 2.1$\times 10^9$ & 28983 & 454 & 732 & 1152 & 1850 & 1428 & 1127 & 5575 & \textbf{\colorbox{dc1}{206}} \\
         &  & Error{\small ($l_2$)} & - & - & 1.23 & 2.19 & 1.97 & 0.09 & 3.95 & 0.99 & 3.83 & \textbf{\colorbox{dc1}{0.08}} \\ \midrule 
         
    
         \multirow{6}{*}{\textbf{\begin{tabular}[c]{@{}c@{}}XGB\\ \end{tabular}}} & \multirow{2}{*}{3} & Time{\small (s)} & 29.2 & 6.5 & 4.7 & 4.6 & 8.5 & 8.2 & \multirow{2}{*}{\textbackslash} & \multirow{2}{*}{\textbackslash} & \multirow{2}{*}{\textbackslash} & \textbf{\colorbox{dc1}{1.8}} \\
         &  & Error{\small ($l_2$)} & - & - & 0.95 & 1.38 & 0.45 & 0.27 & ~ & ~ & ~ & \textbf{\colorbox{dc1}{0.04}} \\ \cmidrule(){2-13} 
         & \multirow{2}{*}{6} & Time{\small (s)} & 13308 & 96 & 19 & 14 & 22 & 25 & \multirow{2}{*}{\textbackslash} & \multirow{2}{*}{\textbackslash} & \multirow{2}{*}{\textbackslash} & \textbf{\colorbox{dc1}{3}} \\
         &  & Error{\small ($l_2$)} & - & - & 0.98 & 2.16 & 1.77 & 0.13 & ~ & ~ & ~ & \textbf{\colorbox{dc1}{0.07}} \\ \cmidrule(){2-13} 
         & \multirow{2}{*}{10} & Time{\small (s)} & 1.7$\times 10^8$ & 2256 & 50 & 81 & 111 & 151 & \multirow{2}{*}{\textbackslash} & \multirow{2}{*}{\textbackslash} & \multirow{2}{*}{\textbackslash} & \textbf{\colorbox{dc1}{5}} \\
         &  & Error{\small ($l_2$)} & - & - & 0.98 & 1.41 & 1.59 & 0.13 & ~ & ~ & ~ & \textbf{\colorbox{dc1}{0.12}} \\ \bottomrule 
        \end{tabular}
        }
        \caption{
        \small
        We mark the ``best performance'' as \textbf{\colorbox{dc1}{green}}. 
        ``\textbackslash'' denotes gradient-based approximation is not applicable to the XGB model.
        }
        \label{tab:real_exp_res_2}
    \end{table}

    \fakeparagraph{Results on Adult} We also take a tabular dataset and adopt the MLP and XGB model as the FL model to compare the effectiveness and efficiency of the baselines.     \tabref{tab:real_exp_res_2} presents the experimental results over different client numbers.
     
    \noindent \textit{In MLP model}.
    The experimental results on Adult are similar to those on FEMNIST, when using MLP as the FL model.
    \textit{(i)} For time cost, \textsl{IPSS} is still the most efficient when there are 10 clients and it is 2.2$\times$ faster than \textsl{DIG}, the second most efficient algorithm.
    \textit{(ii)} For approximation error, \textsl{IPSS} exhibits the lowest error over all client numbers, achieving an average improvement of 43$\times$ over \textsl{GTG-Shap} and 34$\times$ over $\lambda$-\textsl{MR}.

    \noindent\textit{In XGB model.}
    As gradient-based algorithms (\textsl{GTG-Shapley}, \textsl{OR} and $\lambda$-\textsl{MR}) are not applicable to XGBoost, we evaluate the definition-based and sampling-based approaches for calculating the SV in this setup.
    The experimental observations are as follows.
    \textit{(i)} When varying numbers of clients, \textsl{IPSS} consistently shows its superior in efficiency.
    When there are 10 clients in FL, it is 10$\sim$30$\times$ faster than other compared algorithms.
    \textit{(ii)} Similarly, \textsl{IPSS} achieves the lowest approximation error, reducing the error by 25.7$\times$ and 16.6$\times$ compared to \textsl{Extended-TMC} and \textsl{Extended-GTB}, respectively.
    \textit{In this setup, the proposed \textsl{IPSS} performs the best in efficiency as the client number increases and achieves the highest accuracy as well.}

\rev{
    \subsection{In-depth Analysis of Compared Algorithms}
    Next, we conduct more interpretation experiments on the FL benchmark dataset FEMNIST~\cite{arXiv18_LEAF} to validate the efficiency and effectiveness of the compared approximation algorithms.    
}

\begin{figure*}[htb]
    \centering
    \setlength{\tabcolsep}{0.5pt}
    \begin{tabular}{cc}
        \belowrulesep=0.10pt
        \aboverulesep=0.10pt
        \includegraphics[width=0.48\linewidth]{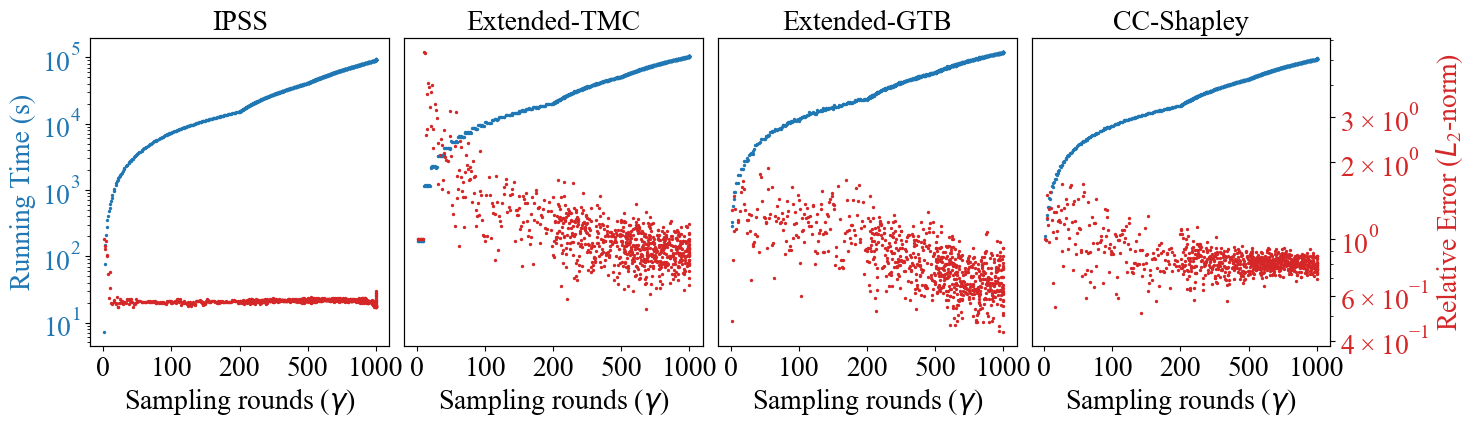} &
        \includegraphics[width=0.48\linewidth]{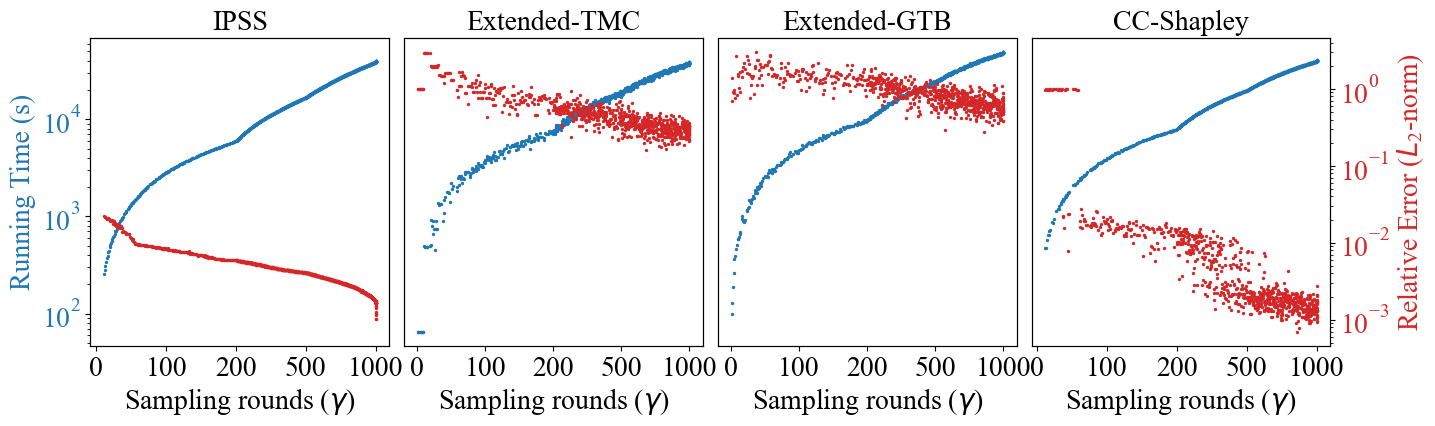}  \\
        \textit{\rev{\small (a) Results of using the MLP as FL model.}} &
        \textit{\rev{\small (b) Results of using the CNN as FL model.}} \\ 
    \end{tabular}
    \vspace{-0.5em}
    \caption{\small \rev{Results on FEMNIST when varying sampling rounds $\gamma$}.}
    \label{fig:VaryRoundFEMNIST}
    \vspace{-1.5em}
\end{figure*}

\begin{figure*}[htb]
    \centering
    \setlength{\tabcolsep}{0.05pt}
    \begin{tabular}{cccccc}
        \belowrulesep=0.10pt
        \aboverulesep=0.10pt
        \includegraphics[width=0.155\linewidth]{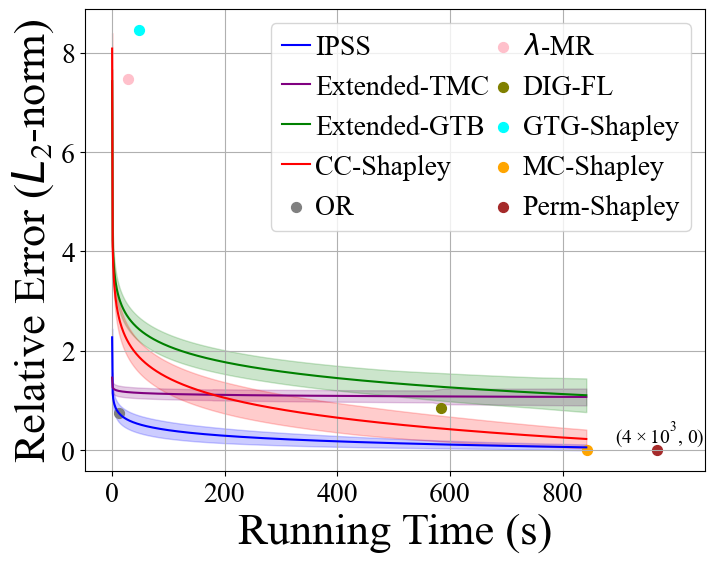} &  \includegraphics[width=0.16\linewidth]{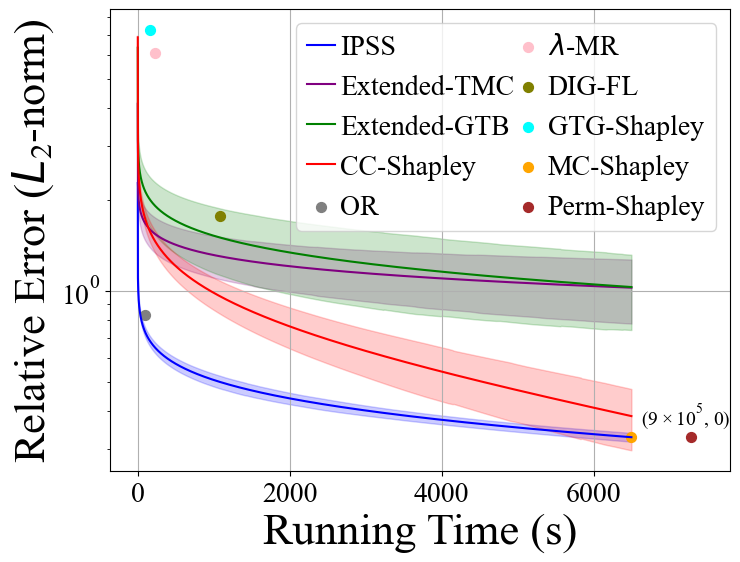} &
        \includegraphics[width=0.16\linewidth]{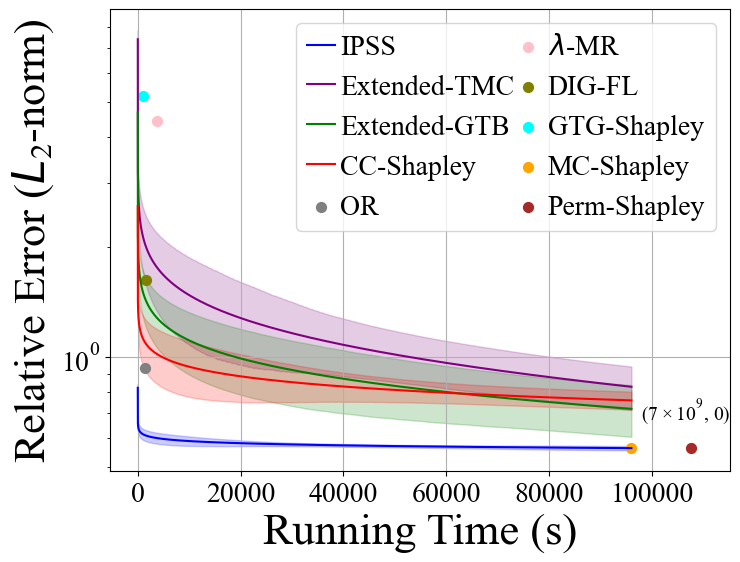} &
        \includegraphics[width=0.155\linewidth]{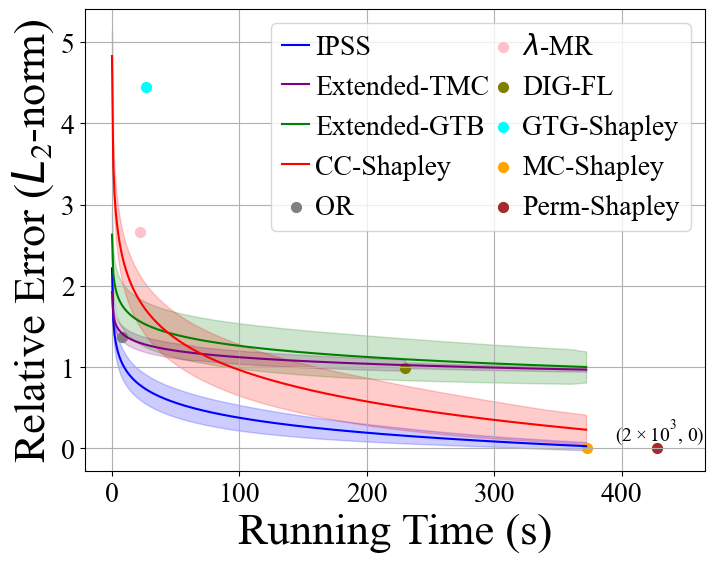}  & 
        \includegraphics[width=0.162\linewidth]{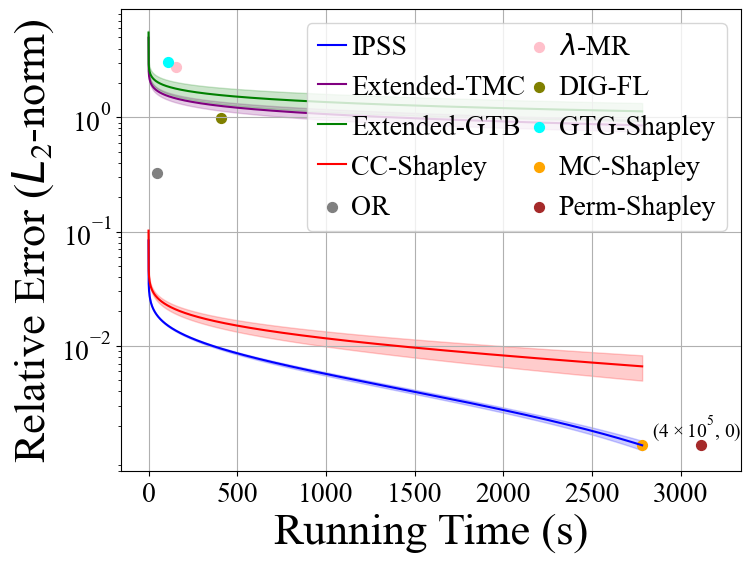}  &
         \includegraphics[width=0.162\linewidth]{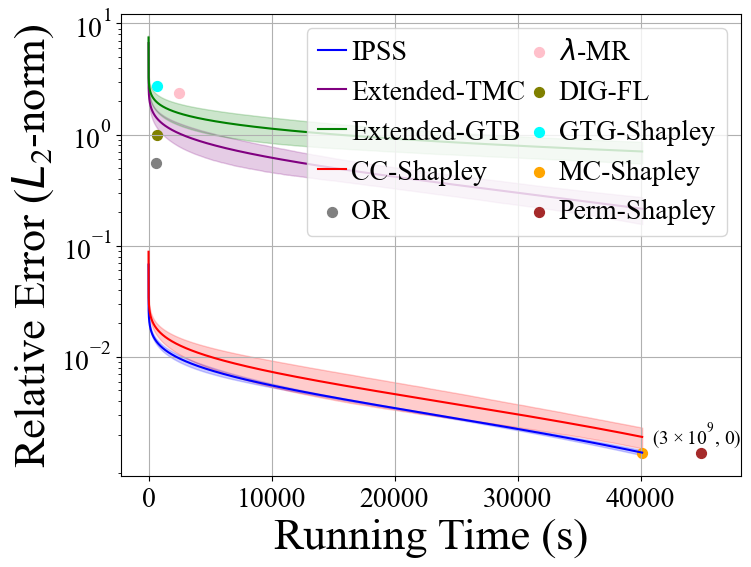} \\
        \textit{\small \rev{(a) \textit{Client\#3$\medop{+}$MLP}}} &  
        \textit{\small \rev{(b) \textit{Client\#6$\medop{+}$MLP}}} & 
        \textit{\small \rev{(c) \textit{Client\#10$\medop{+}$MLP}}} &        
        \textit{\small \rev{(d) \textit{Client\#3$\medop{+}$CNN}}} &
        \textit{\small \rev{(e) \textit{Client\#6$\medop{+}$CNN}}} &
        \textit{\small \rev{(f) \textit{Client\#10$\medop{+}$CNN}}} \\
    \end{tabular}
    \vspace{-0.5em}
    \caption{\small \rev{Pareto curves for trade-off in efficiency and effectiveness.}}
    \label{fig:real_1_pareto}
\end{figure*}

\subsubsection{\rev{Impacts of varying the sampling rounds}} 
\rev{
    As the total sampling round is crucial for sampling-based solutions (\ie \textsl{IPSS}, \textsl{Extended-TMC}, \textsl{Extended-GTB} and \textsl{CC-Shapley}), we study the impacts of varying total sampling rounds $\gamma$ with ten FL clients on FEMNIST~\cite{arXiv18_LEAF}.
    From \figref{fig:VaryRoundFEMNIST}, we have following observations.
    \textit{(i)}  
    As $\gamma$ grows, \textsl{IPSS} has more stable and lower error compared with other baselines.
    Specifically, the variance in error of \textsl{CC-Shapley} is $7.7\times$ and $50.9\times$ higher than that of \textsl{IPSS} on MLP and
    CNN, respectively.
    \textit{(ii)} \textsl{IPSS} fast achieves low approximation errors (\ie below $10^{-2}$) with $\gamma<100$, whereas \textsl{CC-Shapley} reaches the same error level stably only when $\gamma>200$.
    In summary, \textsl{IPSS} achieves lower error more quickly and stably than compared algorithms.
}

\subsubsection{\rev{Pareto curves for time-error trade-off}}
\rev{
    We run the sampling-based algorithms 100 times with each $\gamma$ and plot the \textit{Pareto curves} to show the trade-off between efficiency and effectiveness.
    The experimental results using FEMNIST~\cite{arXiv18_LEAF} across three, six, and ten FL clients are shown in \figref{fig:real_1_pareto} (a)$\sim$(f). 
    We have the following observations:
    \textit{(i)} \textsl{IPSS} achieves \textit{Pareto optimality} on FEMNIST across various numbers of FL clients.
    \textit{(ii)} Though \textsl{OR} runs fast in above end-to-end experiments with three to six clients (in \tabref{tab:real_exp_res_1}), \textsl{IPSS} can achieve comparable performance to \textsl{OR} with a suitable sampling round $\gamma$.
}

    \begin{figure}[h]
        \centering
        \setlength{\tabcolsep}{0.05pt}
        \begin{tabular}{cc}
            \includegraphics[width=0.45\linewidth]{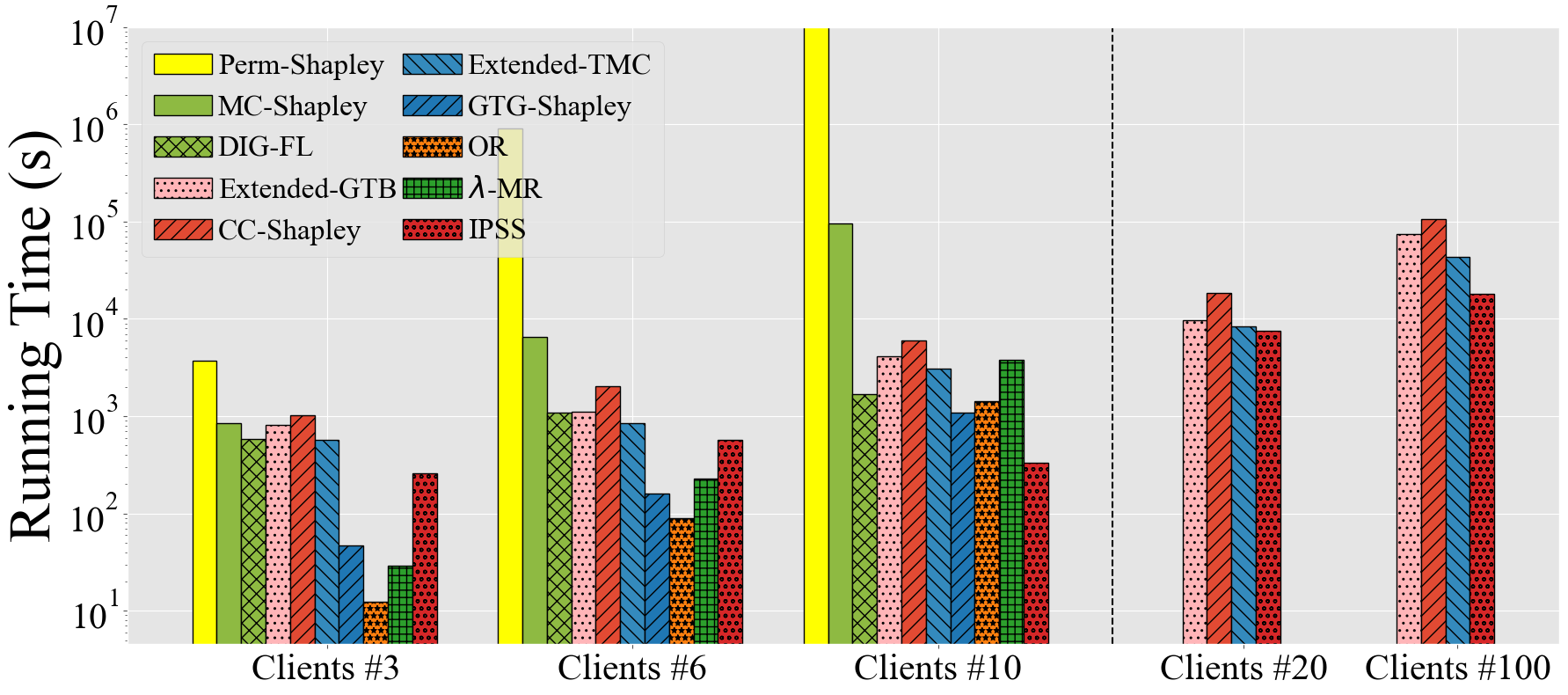} & 
            \includegraphics[width=0.45\linewidth]{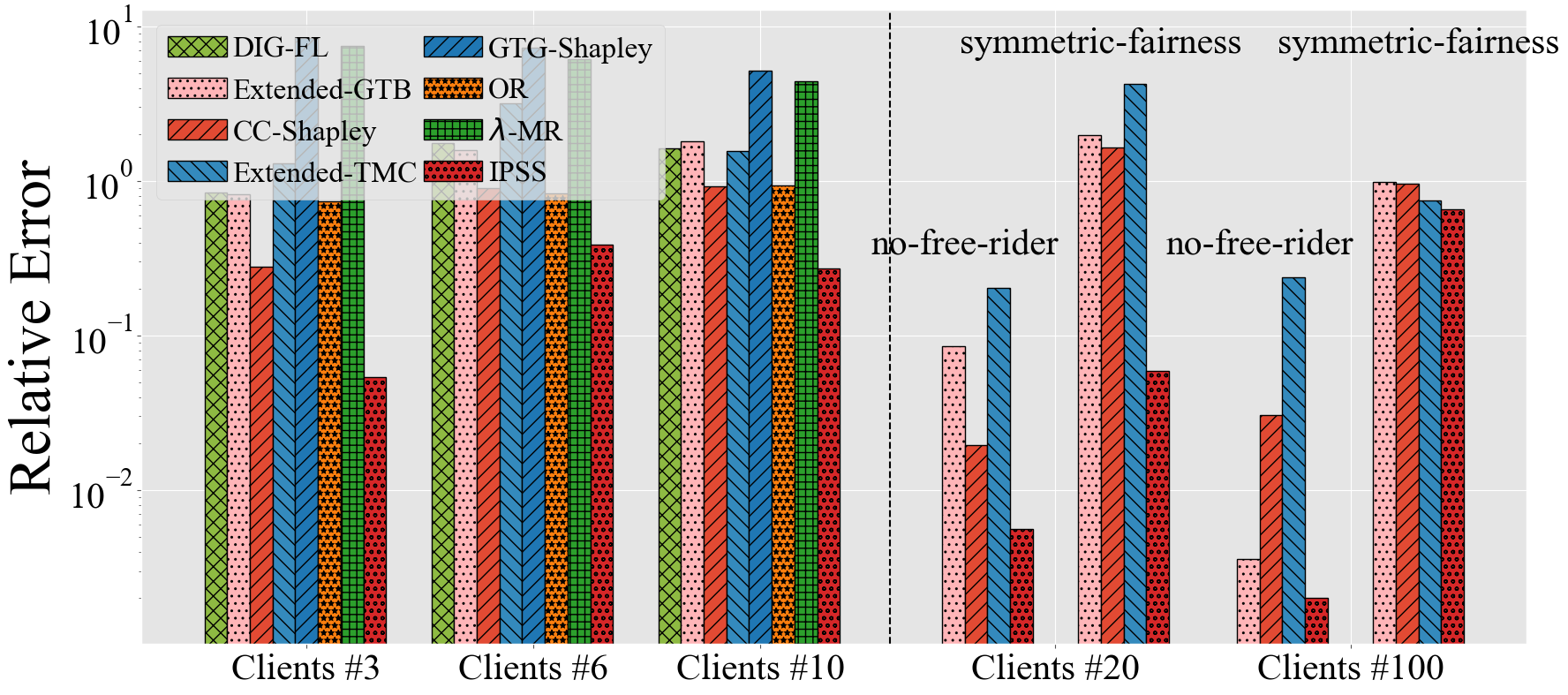}  \\
            \textit{{\small (a) Time cost.}} & \textit{{\small (b) Approximation error.}} 
        \end{tabular}
        \vspace{-0.5em}
        \caption{\small \rev{Varying client number on FEMNIST using MLP model.}}
        \label{fig:VaryNumber}
    \end{figure}

\subsubsection{\rev{Scalability test for larger FL clients}}  
\rev{
        We conduct experiments with up to 100 FL clients, a large-scale scenario for cross-silo FL~\cite{FL_survey_long, CURS24_FL_Survey, yang2019federated}, where more than $10^{30}$ dataset combinations must be assessed by SV definition, making it infeasible to compute the ground-truth within limited time.
        Thus, we set 5\% of FL clients with empty datasets and 5\% of FL clients having same datasets as others and take the extent to which algorithms satisfy required properties (\ie \textit{no-free-rider} and \textit{symmetric-fairness}) as proxies for approximation error.
        We set the sampling round $\gamma$ for sampling-based algorithms to $n \log n$.
        As in \figref{fig:VaryNumber}: \textit{(i)} For running time, \textsl{IPSS} outperforms \textsl{Extended-TMC}, \textsl{Extended-GTB} and \textsl{CC-Shapley} with both $20$ and $100$ FL clients.
        \textit{(ii)} As client number increases from 20 to 100, the running time of our \textsl{IPSS} increases only by $2.4\times$.
        \textit{(iii)} For error based on \textit{no-free-riders} and \textit{symmetric fairness}, \textsl{IPSS} achieves the lowest error among compared algorithms.
}

\subsubsection{\rev{Comparing variance of \textsl{\small MC-SV} and \textsl{\small CC-SV}}} 
\rev{
        We run \algref{alg:mc4sv} 100 times using \textsl{MC-SV} and \textsl{CC-SV}, respectively, to calculate their variance.
        The experimental results are shown in \figref{fig:variance}.
        \textit{(i)} As $\gamma$ increases, the variance of \textsl{MC-SV} and \textsl{CC-SV} initially rises and then decreases as almost all possible combinations are sampled, introducing nearly exact data values with close to zero variance.
        \textit{(ii)} Using both MLP and CNN models, \textsl{MC-SV} exhibits lower variance than \textsl{CC-SV}, with FL clients number from three to ten, consistent with the theoretical analysis of \thmref{thm:CC_Variance} in \secref{sec:framework_comp_scheme} and justifying the selection of \textsl{MC-SV} for our stratified sampling based approximation.
}

\begin{figure}[htb]
    \centering
     \vspace{-0.5em}
    \setlength{\tabcolsep}{0.05pt}
    \begin{tabular}{cccccc}
        \belowrulesep=0.10pt
        \aboverulesep=0.10pt
        \includegraphics[width=0.9\linewidth]{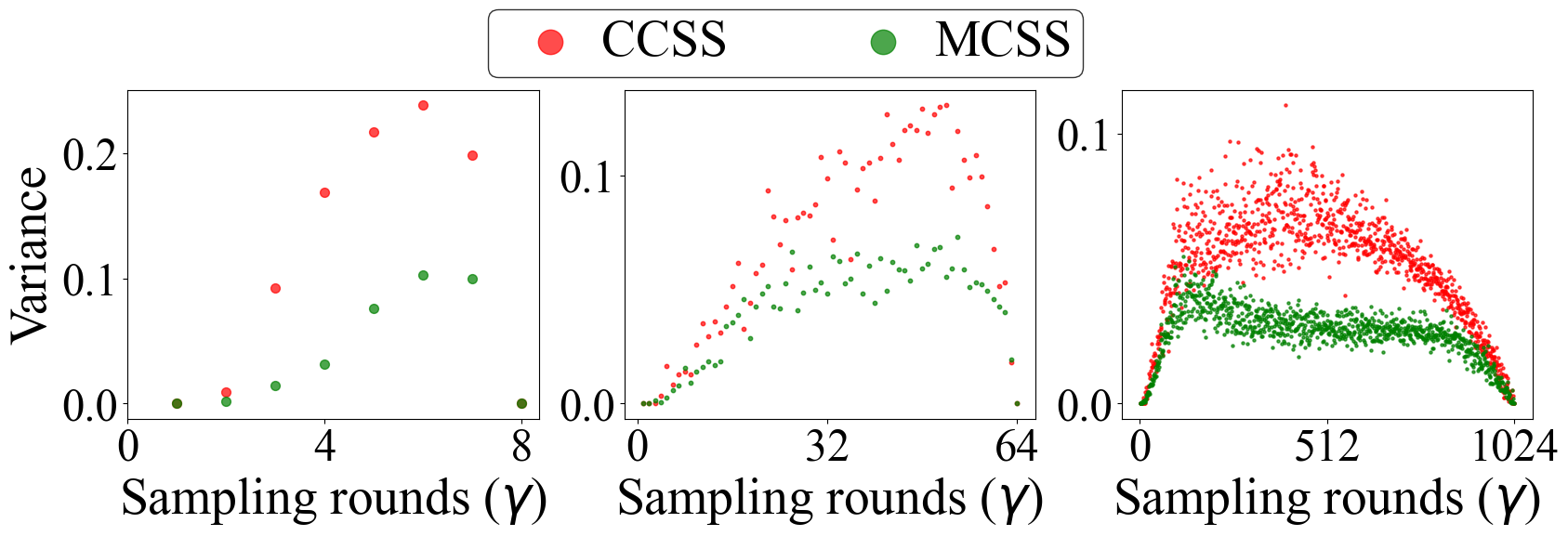} \\
        \vspace{-0.2em}
        \small \rev{(a) Client \#3$\sim$\#10 using MLP model} \\
        
        \includegraphics[width=0.9\linewidth]{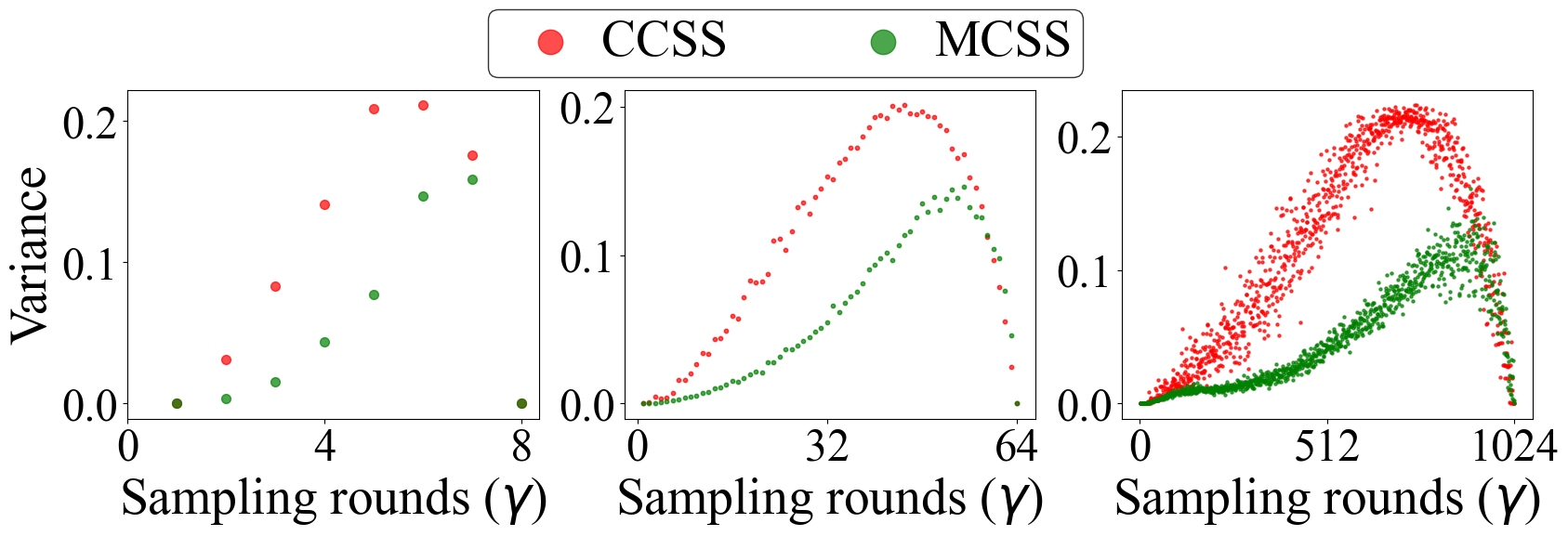} \\
        \small \rev{(b) Client \#3$\sim$\#10 using CNN model}




        %
        
    \end{tabular}
    \vspace{-0.5em}
    \caption{\small \rev{Analysis of variance for \textsl{MC-SV} and \textit{CC-SV}.}}
    \label{fig:variance}
     \vspace{-1.5em}
\end{figure}

\subsection{Summary of Experimental Results}

    \noindent Our major experimental findings are summarized as follows:
    
    \fakeparagraph{Efficiency} 
    Among approximation algorithms evaluated, \textsl{Extended-GTB} and \textsl{CC-Shapley} incur the highest time cost in the most experimental setup.
    Our \textsl{IPSS} algorithm emerges as the most efficient one in most setups, especially within larger number of clients. 
    The time cost of $\lambda$-\textsl{MR} increases exponentially with number of FL clients, limiting its scalability.
    
    \fakeparagraph{Effectiveness} 
    Leveraging insights in \secref{sec:improve_KCP}, the proposed\textsl{IPSS} consistently achieves the lowest estimation errors across nearly all setups.
    Though \textsl{DIG-FL} and \textsl{OR} are more efficient than most compared baselines, they also exhibit a higher approximation error in most experimental setups.


\section{Related Work}\label{sec:related}
    Our work is mainly related to two lines of research topics: the \textit{federated learning} and the \textit{Shapley value based data valuation}.
    We review the representative work in the following.
    
    \subsection{Federated Learning} 
    In recent years, data regulations (\eg GDPR~\cite{gdpr} and CCPA~\cite{ccpa}) have imposed strict requirements on data privacy, posing challenges for privacy-preserving data analysis in both academia and industry.
    \textit{Federated learning (FL)}, enabling multiple data providers to collaboratively train models without sharing their raw data, has emerged as a new paradigm to tackle the data privacy issues.
    Based on the type of clients (\aka data providers), FL can be divided into two settings: \textit{cross-device} and \textit{cross-silo}.
    We introduce each setting below.

    \subsubsection{Cross-Device FL} In this setting, the typical FL clients are a large number mobile or IoT devices\cite{FL_survey_long}, where both the computation and the communication is often the bottleneck.
    Therefore, how to reduce the communication cost is a crucial issue in cross-silo FL.
    In the seminal work~\cite{McMahanMRHA17}, McMahan \etal propose the most widely adopted FL algorithm, \text{FedAVG}, to solve the well-known non-IID problem and reduce communication costs by aggregating model parameters rather than gradients.
    Then, a series of subsequent works have proposed FL algorithms, such as FedProx~\cite{MLSys20}, Scaffold~\cite{ICML20}, \etc
    In addition to the non-IID issue, how to tackle device heterogeneity has recently received increasing attention in cross-silo FL as well~\cite{VLDB23_CDFL, KDD23_CDFL, ICDE22_CDFL, ICDE23_CDFL}.
    For example, authors in~\cite{VLDB23_CDFL, KDD23_CDFL} propose an open-source platforms for real-world cross-device FL, called FS-REAL, which supports advanced FL features such as communication optimization and asynchronous concurrency.
    Liu \etal \cite{KDD22_CDFL} propose the InclusiveFL, an FL framework that adjusts the size of models before assigning them to clients with different computing capabilities.
    \subsubsection{Cross-Silo FL} Yang \etal~\cite{yang2019federated} enrich the concept of FL and introduce the cross-silo FL, a scenario usually involving a small number of clients, such as institutions or companies with abundant computational and communication resources.
    The non-IID issue is also the central challenge in cross-silo FL. 
    Representatively, Huang \etal~\cite{AAAI19_CSFL} adopt the neural networks as the FL model and propose the FedAMP algorithm to solve the non-IID problem and Li \etal~\cite{ICDE22_CSFL} conduct a comprehensive experimental study to compare the performance of various FL algorithms in cross-silo FL.
    Besides, tree-based models have been widely studied in cross-silo FL by prior work \cite{SIGMOD21_FL_Tree, VLDB22_FL_Tree, fu2022blindfl, ICDE22_SV_Wang, AAAI20_FL_Tree}, especially when clients hold the partitioned tabular datasets.
    In this paper, we focus on the cross-silo FL setting and adopt both the neural networks and tree-based models as the FL model in our evaluations.
   
    \subsection{Shapley Value Based Data Valuation}
    The Shapley value~\cite{Shapley1953} has been widely adopted in data valuation \cite{icml19shapley, TIST22_GTG, vldb19shapley, Survey_SV_DB, Survey_SV_ML, ICDE22_SV_Wang, SIGMOD23_SV_Zhang, VLDB23_SV} and some variants are proposed for various scenarios or requirements\cite{FLIP20_FSV, NIPS20_CosShapley, VLDB23_SVF, VLDB24_Shapley, NIPS23_DV, AISTATS23_DV}. 
    Data valuation can be divided into two categories: \textit{within a dataset} and \textit{across datasets}.

    
    \subsubsection{Data Valuation within a Dataset}
    It aims to fairly measure the importance or contribution for each sample (\ie data point).
    In 2019, Ghorbani \etal~\cite{icml19shapley} first introduce the Shapley value in data valuation and define the \textit{Data Shapley} value to qualify the influence of a sample in a dataset.
    To reduce the computational overhead, they further propose two approximation algorithms, \textit{Truncated Monte Carlo} and \textit{Gradient Shapley}, where the former can be extended to the FL framework\cite{Bigdata19_SV}.
    Since computing SV is usually time-consuming, prior work primarily focus on how to design effective and efficient approches to SV based data valuation.
    Jia \etal~\cite{vldb19shapley} propose an exact and efficient algorithm to calculate valuation of samples for \textit{k}NN classification in $\mathcal{O}(n\log{n})$ time complexity, where $n$ is the dataset size.
    They also leverage the sparsity of SV for a singel sample in a dataset to enable efficient approximation \cite{aistats19shapley}. 
    However, the sparsity of SV is inexistent in cross-silo FL, so we only extend their another sampling-based algorithm, Group Testing Based SV, as a baseline in our paper.   
    Zhang \etal \cite{SIGMOD23_SV_Zhang} propose a novel equavilent expression of SV based on complementary contribution, upon which they design a sampling-based approximation algorithm for SV that is also applicable for data valuation.
    We compare their proposed equavilent SV expression (referred as \textit{CC-SV} in this paper) with another two commonly used SV expression and adopt the approaches in \cite{SIGMOD23_SV_Zhang} as one of our baselines as well.
    
    \subsubsection{Data Valuation Across Datasets}
    It aims to identify the contribution of each dataset (\ie the contribution estimation), which is consistent with the data valuation in FL \cite{ICDE22_SV_Wang, VLDB24_Shapley, ICDE24_Shapley}, where how to design an efficient and effective approximation algorithms is the primary issue.
    The typical algorithms in this setup is the gradient construction based approximation, which utilizes gradients in FL to build federated models under various dataset combinations and avoids the need to train extra FL models for data valuation.
    Song \etal~\cite{Bigdata19_SV} first propose the gradient construction based approaches to measures the contribution of datasets in FL and they also propose another algorithm, $\lambda$-MR, to further reduce time cost by reconstructing FL model based on gradients in each training round~\cite{FLIP20_Wei}.
    In \cite{TIST22_GTG}, the authors propose an efficient algorithm called \textit{GTG-Shapley} to approximate the SV by combining the on the gradient construction with Monte Carlo sampling.
    We compare our proposed algorithm against OR, $\lambda$-MR and GTG-Shapley in experimental evaluations.
    Besides, Wang \etal~\cite{ICDE22_SV_Wang} propose an efficient data valuation approaches to measure the contributions of clients in FL, which only needs linear number of evaluations under certain assumption and we take their approach as the baseline as well.
    Zheng \etal \cite{VLDB23_SV} study the secure data valuation for cross-silo FL and exploring ways to enhance the efficiency using an efficient two-server protocol.
    As security is outside the scope of this work, we do not take their approach as one of the baselines in our experiments.

\section{Conclusion} \label{sec:conclusion}
In this paper, we investigate the Shapley value based data valuation in FL and introduce an efficient and effective sampling-based approximation algorithm, \textsl{IPSS}.
Specifically, we first propose a unified stratified sampling-based approximation framework that seamlessly integrates both \textsl{\small MC-SV}-based and \textsl{\small CC-SV}-based computation schemes.
We also identify a crucial phenomenon called {key combinations}, where only limited dataset combinations highly impact final data value results in FL.
Building upon our new findings, we propose a practical approximation algorithm, \textsl{IPSS}, which strategically selects high-impact dataset combinations rather than taking all possible dataset combinations in FL, thus significantly improving the efficiency with high approximation accuracy.
Finally, we conduct extensive evaluations on real and synthetic datasets to validate that the proposed \textsl{IPSS} outperforms the representative baselines in both efficiency and effectiveness. 



\bibliographystyle{IEEEtran}
\bibliography{arxiv}



\end{document}